\documentclass[11pt]{article}
\usepackage{graphicx}
\newcount\Comments
\usepackage{bbm}
\usepackage{amsmath}
\usepackage{amssymb}
\usepackage{algorithm}
\usepackage[framemethod=TikZ]{mdframed}
\usepackage{mathtools}
\usepackage{amsthm}
\usepackage{natbib}
\usepackage{letltxmacro}
\usepackage{authblk}

\PassOptionsToPackage{algo2e,ruled, linesnumbered}{algorithm2e}
\RequirePackage{algorithm2e}

\LetLtxMacro\amsproof\proof
\LetLtxMacro\amsendproof\endproof

\usepackage{thmtools}
\usepackage{xcolor}
\AtBeginDocument{%
  \LetLtxMacro\proof\amsproof
  \LetLtxMacro\endproof\amsendproof
}

\newtheorem{theorem}{Theorem}

\newtheorem{definition}{Definition}
\newtheorem{lemma}[theorem]{Lemma}
\newtheorem{corollary}[theorem]{Corollary}
\newtheorem{conjecture}[theorem]{Conjecture}

\newtheorem{example}[theorem]{Example}
\newtheorem*{theorem*}{Theorem}
\newtheorem*{remark*}{Remark}

\usepackage{xpatch}
\xpatchcmd{\proof}{\itshape}{\normalfont\proofnamefont}{}{}

\usepackage[most]{tcolorbox}

\usepackage[letterpaper,top=2cm,bottom=2cm,left=3cm,right=3cm,marginparwidth=1.75cm]{geometry}

\usepackage[colorlinks=true, allcolors=blue]{hyperref}

\DeclarePairedDelimiter\ceil{\lceil}{\rceil}

\newcommand{\naturals}{\mathbb{N}}

\usepackage{color}
\definecolor{darkgreen}{rgb}{0,0.5,0}
\definecolor{purple}{rgb}{1,0,1}
\newcommand{\kibitz}[2]{\ifnum\Comments=1\textcolor{#1}{#2}\fi}

\newcommand{\Acal}{\mathcal{A}}
\newcommand{\Hcal}{\mathcal{H}}
\newcommand{\Tcal}{\mathcal{T}}

\newcommand{\Xcal}{\mathcal{X}}
\newcommand{\Ycal}{\mathcal{Y}}

\newcommand{\Scal}{\mathcal{S}}

\DeclareMathOperator*{\argmin}{arg\,min}
\newcommand{\expect}{\operatorname{\mathbb{E}}}

\newcommand{\indicator}{\mathbbm{1}}

\newcommand*\samethanks[1][\value{footnote}]{\footnotemark[#1]}

\title{Online Learning with Set-valued Feedback}

\author{Vinod Raman\thanks{Authors contributed equally}, Unique Subedi\samethanks , and Ambuj Tewari}
\affil{Department of Statistics, University of Michigan}
\affil{\texttt{\{vkraman, subedi, tewaria\}@umich.edu}}

\date{}

\begin{document}

\maketitle

\begin{abstract}%

\noindent We study a variant of online multiclass classification where the learner predicts a single label but receives a \textit{set of labels} as feedback. In this model, the learner is penalized for not outputting a label contained in the revealed set. We show that unlike online multiclass learning with single-label feedback, deterministic and randomized online learnability are \textit{not equivalent} even in the realizable setting with set-valued feedback. Accordingly, we give two new combinatorial dimensions, named the Set Littlestone and Measure Shattering dimension, that tightly characterize deterministic and randomized online learnability respectively in the realizable setting. In addition, we show that the Measure Shattering dimension characterizes online learnability in the agnostic setting and tightly quantifies the minimax regret. Finally, we use our results to establish bounds on the minimax regret for three practical learning settings:  online multilabel ranking,  online multilabel classification, and real-valued prediction with interval-valued response.
\end{abstract}


\section{Introduction} 

In the standard online multiclass classification setting, a learner plays a repeated game against an adversary. In each round $t \in [T]$, the adversary picks a labeled example $(x_t, y_t) \in \mathcal{X} \times \mathcal{Y}$ and reveals the unlabeled example $x_t$ to the learner. The learner observes $x_t$ and then makes a prediction $\hat{y}_t \in \mathcal{Y}$. Finally, the adversary reveals the true label $y_t$ and the learner suffers the loss $\mathbbm{1}\{\hat{y}_t \neq y_t\}$ \citep{Littlestone1987LearningQW, DanielyERMprinciple}.

In practice, however, there may not be a single correct label $y\in \mathcal{Y}$, but rather, a \textit{collection} of correct labels $S\subseteq \mathcal{Y}$. For example, in online multilabel ranking,  the learner is tasked with ranking a set of labels in terms of their relevance to an instance. However, as feedback, the learner only receives a bitstring indicating which of the labels were relevant. This feedback model is standard in multilabel ranking since obtaining the full ranking is generally costly \citep{liu2009learning}. Since, for any given bitstring, there can be multiple rankings that correctly place relevant labels above non-relevant labels, the learner effectively only observes a \textit{set} of correct rankings. Beyond ranking, other notable examples of set-valued feedback include multilabel classification with a thresholded Hamming loss, where the learner is only penalized after misclassifying a certain number of labels, and real-valued prediction where the response is an interval on the real line \citep{diamond1990least, gil2002least, huber2009interval}. Even more generally, one can equivalently represent the ground truth label as a collection of elements from the prediction space for any learning problem with the 0-1 loss where there is an asymmetry between the prediction and label space. 

Motivated by online multilabel ranking and other natural learning problems, we study a variant of online multiclass classification where in each round $t \in [T]$, the learner still predicts a single label $\hat{y}_t \in \mathcal{Y}$, but the adversary reveals a set of correct labels $S_t \in \mathcal{S}(\mathcal{Y})$, where $\mathcal{S}(\mathcal{Y})\subseteq 2^{\mathcal{Y}}$ is an arbitrary set system. The learner suffers a loss if and only if $\hat{y}_t \notin S_t$. Given a hypothesis class $\mathcal{H} \subseteq \mathcal{Y}^{\mathcal{X}}$, the goal of the learner is to output predictions such that its regret, the difference between its cumulative loss and the cumulative loss of the best-fixed hypothesis in hindsight, is small. The class $\mathcal{H}$ is said to be online learnable if there exists an online learning algorithm whose regret is a sublinear function of the time horizon $T$. 

Given a learning problem $(\mathcal{X}, \mathcal{Y}, \mathcal{S}(\mathcal{Y}), \mathcal{H})$, what are necessary and sufficient conditions for $\mathcal{H}$ to be online learnable? For example, under single-label feedback (multiclass classification), the online learnability of a hypothesis class $\mathcal{H} \subseteq \mathcal{Y}^{\mathcal{X}}$ is characterized by the finiteness of a combinatorial parameter called the Littlestone dimension \citep{Littlestone1987LearningQW,  ben2009agnostic, DanielyERMprinciple}.  Analogously, is there a combinatorial parameter that characterizes online learnability under set-valued feedback? Motivated by these questions, we make the following contributions.


\begin{itemize}
    \item[(1)] We show that under set-valued feedback, deterministic and randomized learnability are \textit{not equivalent} even in the realizable setting. This is in contrast to online learning with single-label feedback, where there is no separation between deterministic and randomized realizable learnability \citep{Littlestone1987LearningQW, DanielyERMprinciple}. Additionally, we show deterministic and randomized realizable learnability are equivalent if the \emph{Helly number}, a parameter that arises in combinatorial geometry, of $\mathcal{S}(\mathcal{Y})$ is finite. 

    \item[(2)] In light of this separation, we give two new combinatorial dimensions,  the Set Littlestone and Measure shattering dimension,  and show that they characterize deterministic and randomized realizable learnability respectively.

    \item[(3)] Moving beyond the realizable setting, we show that the Measure Shattering dimension continues to characterize \textit{agnostic} learnability. This implies an equivalence between randomized realizable learnability and agnostic learnability. 
    \item[(4)] Finally, as applications, we use our results to bound the minimax expected regret for three practical learning settings:  online multilabel ranking, online multilabel classification, and real-valued prediction with interval-valued response. 
\end{itemize}

To prove the separation in (1), we identify a learning problem where every deterministic learner fails, but there exists a simple randomized learner. As for our combinatorial dimensions in (2),  the Set Littlestone and Measure shattering dimensions are defined using complete trees with \textit{infinite-width}. This is in contrast to much of the existing combinatorial dimensions in online learning. To prove that the Set Littlestone dimension is sufficient for deterministic realizable learnability, we extend the Standard Optimal Algorithm for single-label to set-valued feedback. On the other hand, to prove that the Measure shattering dimension is sufficient for randomized realizable learnability, we adapt the recent algorithmic chaining technique from \cite{daskalakis2022fast}. Lastly, our construction of an agnostic learner in (3) uses a non-trivial extension of the adaptive covering technique introduced in \cite{hanneke2023multiclass}.

\subsection{Related Works}
There is a rich history of characterizing online learnability in terms of combinatorial dimensions. For example, \cite{Littlestone1987LearningQW, ben2009agnostic} proved that the Littlestone dimension characterizes online learnability in binary classification. Studying optimal randomized learnability, \cite{filmus2023optimal} proposed the Randomized Littlestone and showed that it characterizes optimal regret bounds for randomized learners in the realizable setting. \cite{DanielyERMprinciple, hanneke2023multiclass} show that the Littlestone dimension continues to characterize online learnability in the multiclass classification setting. Recent work by \cite*{moran2023list} showed that a modification of the Littlestone dimension characterizes \textit{list online classification}, the ``flip" of our setting where the learner outputs a set of labels, but the adversary reveals a single label. In addition, \cite{daniely2013price} showed that the Bandit Littlestone dimension characterizes online learnability when the adversary can output a set of correct labels, however, the learner only observes the indication of whether their predicted label was in the set or not. 
Moreover, there is a growing literature on online multiclass learning with feedback graphs \citep{van2021beyond, alon2015online}. In this setting, the learner predicts a single label but observes the losses of a specific set of labels determined by an arbitrary directed feedback graph. 
Finally, the Helly number \cite{helly1923mengen} has previously been used to characterize proper learning in both online and PAC settings \citep{hanneke2021online, pmlr-v125-bousquet20a} and has also appeared in the literature on distributed learning \citep{kane2019communication}.

\subsection{Relation to List Online Classification}
List online classification, studied by \cite{moran2023list}, is intimately related to online classification with set-valued feedback. Indeed, online classification with set-valued feedback is equivalent to a modified list online classification game, where in each round $t \in [T]$:  (1) the learner picks a label $\hat{y}_t \in \mathcal{Y}$  and constructs a list $\hat{L}_t \subset \mathcal{S}(\mathcal{Y})$ such that $\hat{y}_t \in S$ for every $S\in \hat{L}_t$, (2) Nature reveals the true set $S_t \in S(\mathcal{Y})$, and (3) the learner suffers the loss $\mathbbm{1}\{S_t \notin \hat{L}_t\} \geq \mathbbm{1}\{\hat{y}_t \notin S_t\}$. However, there are important differences between this ``modified" list online classification game and the ``original" list online classification game proposed by \cite{moran2023list} when taking $S(\mathcal{Y})$ to be the label space. First, in the ``original" list online classification game, the learner is allowed to output \emph{any} finite list of elements in $S(\mathcal{Y})$. This is not the case with the ``modified" list online classification game. Indeed, the ``modified" list online learner is required to pick any sequence of elements in $S(\mathcal{Y})$ whose sequence-wise intersection is not empty. This means that the ``modified" list online classification game can be harder than the ``original" list online classification game, for example, when $S(\mathcal{Y})$ contains all disjoint sets. On the other hand, the  ``original" list online classification game can also be harder than the ``modified" list online classification game, for example, when $\bigcap_{S \in S(\mathcal{Y})}S \neq \emptyset$. These statements are true even when the sets $S_t \in S(\mathcal{Y})$ are all finite. Therefore, the ``modified" and ``original" list online classification game with label space $S(\mathcal{Y})$ are incomparable.

\section{Preliminaries}
\subsection{Notation}
Let $\Xcal$ denote the instance space and $(\Ycal, \sigma(\Ycal))$ be a measurable label space. Let $\Pi(\Ycal)$ denote the set of all probability measures on ($\Ycal, \sigma(\Ycal))$.  In this paper, we consider the case where $\mathcal{Y}$ can be unbounded (e.g. $\mathcal{Y} = \mathbbm{N}$). Given a measurable label space $(\Ycal, \sigma(\Ycal))$, let $\mathcal{S}(\mathcal{Y}) \subseteq \sigma(\Ycal)$ denote an arbitrary, measurable collection of subsets of $\mathcal{Y}$. For any set $S \in \mathcal{S}(\mathcal{Y})$, we let $S^c = \mathcal{Y} \setminus S$ denote its complement. Let $\mathcal{H} \subseteq \mathcal{Y}^{\mathcal{X}}$ denote an arbitrary hypothesis class consisting of predictors $h: \mathcal{X} \rightarrow \mathcal{Y}$. Finally, we let $[N] := \{1, 2, \ldots, N\}$.

\subsection{Online Learning} \label{sec:prelim_onl}
In the online setting, an adversary plays a sequential game with the learner over $T$ rounds. In each round $t \in [T]$, an adversary selects a labeled instance $(x_t, S_t) \in \mathcal{X} \times \mathcal{S}(\mathcal{Y})$ and reveals $x_t$ to the learner. The learner makes a potentially randomized prediction $\hat{y}_t \in \mathcal{Y}$. Finally, the adversary reveals the set $S_t$, and the learner suffers the loss $\mathbbm{1}\{\hat{y}_t \notin S_t\}$. Given a hypothesis class  $\mathcal{H} \subseteq \mathcal{Y}^{\mathcal{X}}$, the goal of the learner is to output predictions $\hat{y}_t$ such that its cumulative loss is close to the best possible cumulative loss over hypotheses in $\mathcal{H}$. Before we define online learnability, we provide formal definitions of deterministic and randomized online learning algorithms. 


\begin{definition}[Deterministic Online Learner]{}
\noindent A deterministic online  learner is a deterministic mapping $\Acal : (\Xcal \times \Scal(\Ycal))^{\star} \times \Xcal \to \Ycal$ that maps past examples  and the newly revealed instance $x \in \Xcal$ to a label $y \in \Ycal$. 
\end{definition}

\begin{definition}[Randomized Online Learner]{}
 \noindent A randomized online learner is a \emph{deterministic} mapping $\Acal : (\Xcal \times \Scal(\Ycal))^{\star} \times \Xcal \to \Pi(\Ycal)$ that maps past examples  and the newly revealed instance $x \in \Xcal$ to a probability distribution $\hat{\mu} \in \Pi(\Ycal)$. The learner then randomly samples a label $\hat{y} \sim \hat{\mu}$ to make a prediction.
\end{definition}

\noindent We typically use $\Acal(x)$ to denote the prediction of $\Acal$ on $x$. When $\mathcal{A}$ is randomized, we  use $\Acal(x)$ to denote the random sample $\hat{y}$ drawn from the distribution that $\Acal$ outputs.


A hypothesis class is said to be online learnable if there exists an online learning algorithm, either deterministic or randomized, whose (expected) cumulative loss, on any sequence of labeled examples, $(x_1, S_1), ..., (x_T, S_T)$, is not too far from that of best-fixed hypothesis in hindsight.

\begin{definition} [Online Agnostic Learnability]\label{def:agnOL} \noindent 
 A hypothesis class $\Hcal \subseteq \Ycal^{\Xcal}$ is online learnable in the agnostic setting if there exists a (potentially randomized) algorithm $\mathcal{A}$ such that 
 its \emph{expected regret}
$$\emph{\texttt{R}}_{\mathcal{A}}(T, \mathcal{H}) := \sup_{(x_1, S_1), ..., (x_T, S_T)} \left(\mathbb{E}\left[\sum_{t=1}^T \mathbbm{1}\{\mathcal{A}(x_t) \notin S_t\}\right] - \inf_{h \in \mathcal{H}}\sum_{t=1}^T  \mathbbm{1}\{h(x_t) \notin  S_t \} \right)$$

\noindent is a non-decreasing, sub-linear function of $T$. 
\end{definition}



\noindent  A sequence of labeled examples $\{(x_t, S_t)\}_{t=1}^T$ is said to be \textit{realizable} by $\Hcal$ if there exists a hypothesis $h^{\star} \in \Hcal$ such that $h^{\star}(x_t) \in S_t$ for all $t \in [T]$.
In such case, we have $ \inf_{h \in \Hcal} \sum_{t=1}^T  \mathbbm{1}\{h(x_t) \notin  S_t \} = 0$. 

\begin{definition}[Online Realizable Learnability]\label{realOL}
\noindent A hypothesis class $\Hcal \subseteq \Ycal^{\Xcal}$ is online learnable in the realizable setting if there exists a (potentially randomized) algorithm $\mathcal{A}$ such that 
its \emph{expected number of mistakes} 
$$\emph{\texttt{M}}_{\mathcal{A}}(T, \mathcal{H}) := \sup_{\substack{(x_1, S_1), ..., (x_T, S_T)\\ \exists h^{\star} \in \mathcal{H} \text{ such that } h^{\star}(x_t) \in S_t} }\mathbb{E}\left[\sum_{t=1}^T \mathbbm{1}\{\mathcal{A}(x_t) \notin S_t\}\right]$$
is a non-decreasing, sub-linear function of $T$. 
\end{definition}

One may analogously define a slightly restricted notion of deterministic realizable learnability by restricting the algorithm $\Acal$ to be deterministic. 

\section{Combinatorial Dimensions}

In online learning theory, combinatorial dimensions are often defined in terms of \textit{trees}, a basic unit that captures temporal dependence. Accordingly, we start this section by formally defining the notion of a tree. 

 Given an instance space $\mathcal{X}$ and a (potentially uncountable) set of objects $\mathcal{M}$, an $\mathcal{X}$-valued, $\mathcal{M}$-ary tree $\mathcal{T}$ of depth $T$ is a complete rooted tree such that each internal node $v$ is labeled by an instance $x \in \mathcal{X}$ and  for every internal node $v$ and object $m \in \mathcal{M}$, there is an outgoing edge $e^m_{v}$  indexed by $m$. We can mathematically represent this tree by a sequence $(\mathcal{T}_1, ..., \mathcal{T}_T)$ of labeling functions $\mathcal{T}_t:\mathcal{M}^{t-1} \rightarrow \mathcal{X}$ which provide the labels for each internal node. A path of length $T$ down the tree is given be a sequence of objects $m = (m_1,..., m_T) \in \mathcal{M}^T$. Then, $\mathcal{T}_t(m_1, ..., m_{t-1})$ gives the label of the node by following the path $(m_1, ..., m_{t-1})$ starting from the root node, going down the edges indexed by the $m_t$'s.  We let $\mathcal{T}_1 \in \mathcal{X}$ denote the instance labeling the root node. For brevity, we define $m_{<t} = (m_1, ..., m_{t-1})$ and therefore write $\mathcal{T}_t(m_1, ..., m_{t-1}) = \mathcal{T}_t(m_{<t})$. Analogously, we let $m_{\leq t} = (m_1, ..., m_{t})$.

Often, it is useful to label the edges of a tree with some \textit{auxiliary} information. Given an $\mathcal{X}$-valued, $\mathcal{M}$-ary tree $\mathcal{T}$ of depth $T$ and a (potentially uncountable) set of objects $\mathcal{N}$, we can formally label the edges of $\mathcal{T}$ using objects in $\mathcal{N}$ by considering a sequence $(f_1, ..., f_T)$ of edge-labeling functions  $f_t: \mathcal{M}^{t} \rightarrow \mathcal{N}$. For each depth $t \in [T]$, the function $f_t$ takes as input a path $m_{\leq t}$ of length $t$ and outputs an object in $\mathcal{N}$. Accordingly, we can think of the object $f_t(m_{\leq t})$ as labeling the edge indexed by $m_t$ after following the path $m_{< t}$ down the tree. We now use this notation to rigorously define existing combinatorial dimensions in online learning.

We begin with the Littlestone dimension, which is known to characterize binary/multiclass online classification, where $\Scal(\Ycal) =\{\{y\} : y \in \Ycal\}.$ 

\begin{definition} [Littlestone dimension \citep{Littlestone1987LearningQW, DanielyERMprinciple}]\label{def:ldim}
\noindent Let $\mathcal{T}$ be a complete, $\mathcal{X}$-valued , $\{\pm1\}$-ary tree of depth $d$. The tree $\mathcal{T}$ is shattered by $\mathcal{H} \subseteq \Ycal^{\Xcal}$  if there exists a sequence $(f_1, ..., f_d)$ of edge-labeling functions  $f_t: \{\pm 1\}^{t} \rightarrow \mathcal{Y}$  such that for every path $\sigma = (\sigma_1, ..., \sigma_d) \in \{\pm 1\}^d$, there exists a hypothesis $h_{\sigma} \in \mathcal{H}$ such that for all $t \in [d]$,  $h_{\sigma}(\mathcal{T}_t(\sigma_{<t})) = f_t(\sigma_{\leq t})$ and $f_t((\sigma_{< t}, -1)) \neq f_t((\sigma_{< t}, +1))$. The Littlestone dimension of $\mathcal{H}$, denoted $\emph{\texttt{L}}(\mathcal{H})$, is the maximal depth of a tree $\mathcal{T}$ that is shattered by $\mathcal{H}$. If there exists shattered trees of arbitrarily large depth, we say $\emph{\texttt{L}}(\mathcal{H}) = \infty$.
\end{definition}


A natural extension of the Littlestone dimension to set-valued feedback is to (1) replace the two differing labels on the edges of the Littlestone tree with two disjoint sets in $\mathcal{S}(\mathcal{Y})$ and (2) require that for every path down the tree, there is a hypothesis whose outputs on the sequence of instances lie inside the sets labeling the sequence of edges.  In fact, one can even consider trees with more than two outgoing edges. 
Such combinatorial structures have been previously studied to characterize online learnability under bandit feedback \citep{daniely2013price} and list classification \citep{moran2023list}.

Along this direction, Definition \ref{def:psldim} considers complete trees where each internal node has $p$ outgoing edges.  Each outgoing edge is labeled by a set in $\mathcal{S}(\mathcal{Y})$ with the additional constraint that the mutual intersection of the $p$ sets labeling the $p$ edges has to be empty. Finally, such a $[p]$-ary is shattered if for every root-to-leaf path down the tree, there exists a hypothesis whose outputs on the sequence of instances lie in the sets labeling the edges along the sequence. 

\begin{definition}[$p$-Set Littlestone dimension]\label{def:psldim}
\noindent Let $\mathcal{T}$ be a complete $\mathcal{X}$-valued, $\left[p\right]$-ary tree of depth $d$. The tree $\mathcal{T}$ is shattered by $\mathcal{H} \subseteq \Ycal^{\Xcal}$  if there exists a sequence $(f_1, ..., f_d)$ of edge-labeling set-valued functions  $f_t: \left[p\right]^{t} \rightarrow \mathcal{S}(\mathcal{Y})$  such that for every path $q = (q_1, ..., q_d) \in \left[p\right]^d$, we have $\bigcap_{i \in [p]} f_t((q_{<t}, i)) = \emptyset$ and there exists a hypothesis $h_{q} \in \mathcal{H}$ such that $h_{q}(\mathcal{T}_t(q_{<t})) \in f_t(q_{\leq t})$ for all $t \in [d]$,. The $p$-Set Littlestone dimension of $\mathcal{H}$  denoted $\emph{\texttt{SL}}_{p}(\mathcal{H}, \mathcal{S}(\mathcal{Y}))$, is the maximal depth of a tree $\mathcal{T}$ that is shattered by $\mathcal{H}$. If there exists shattered trees of arbitrarily large depth, we say $\emph{\texttt{SL}}_{p}(\mathcal{H}, \mathcal{S}(\mathcal{Y})) = \infty$. 
\end{definition}

 When it is clear from context, we drop the dependence of $\Scal(\Ycal)$ and only write $\texttt{SL}_p(\mathcal{H})$. Note that if $p_1 > p_2$, then $\texttt{SL}_{p_1}(\mathcal{H}) \geq \texttt{SL}_{p_2}(\mathcal{H})$. It is not too hard to see that the finiteness of $\texttt{SL}_{p}(\mathcal{H})$ for every $p \geq 2$ is a necessary condition for online learnability. For many natural problems (see Theorem \ref{thm:relation} and Section \ref{sec:app}), the finiteness of $\texttt{SL}_{p}(\mathcal{H})$ for every $p \geq 2$ is also sufficient for online learnability. However, Example \ref{exm:counter} shows that the finiteness of $\texttt{SL}_{p}(\mathcal{H})$ for every $p \geq 2$ is actually not sufficient.

 
 \begin{example} \label{exm:counter} \noindent Let $\mathcal{Y} = \mathbb{N}$,  $\mathcal{S}(\mathcal{Y}) = \{A^c: A \subset \mathbb{N}, |A| < \infty\}$, and suppose $\mathcal{H} = \{x \mapsto y: y \in \mathcal{Y}\}$ is the class of constant functions. First, we claim that $\emph{\texttt{SL}}_{p}(\mathcal{H}) = 0$ for all $p \geq 2$. Fix $p \geq 2$ and let  $S_1, ..., S_p \in \mathcal{S}(\mathcal{Y})$ denote an arbitrary sequence of $p$ sets.   For each $i \in [p]$, let $A_i$ be the finite set such that $S_i = A^c_i$. Then, $\bigcap_{i=1}^p S_i  = \bigcap_{i=1}^p A_i^c = \left(\bigcup_{i=1}^p A_i\right)^c \neq \emptyset$ since $|\bigcup_{i=1}^p A_i| < \infty$. Thus, $\emph{\texttt{SL}}_p(\mathcal{H}) = 0$  because it is not possible to find $p$ sets in $\mathcal{S}(\mathcal{Y})$ whose mutual intersection is empty. Since $p$ is arbitrary, this is true for every $p \geq 2$. Next, we claim that $\mathcal{H}$ is not online learnable. This follows from the fact that for every $\varepsilon \in [0, 1]$ and measure $\mu \in \Pi(\mathcal{Y})$, there exists a finite set $A_{\mu} \subset \mathbb{N}$ such that $\mu(A_{\mu}) \geq \varepsilon$. Suppose for the sake of contradiction this is not true. That is, there exists an $\varepsilon \in [0, 1]$ and a measure $\mu_{\varepsilon} \in \Pi(\mathcal{Y})$ such that for all finite sets $A \subset \mathbb{N}$, we have $\mu_{\varepsilon}(A)  < \varepsilon$. For every $i \in \mathbb{N}$, let $N_i = \{1, 2, ..., i\}$ denote the first $i$ natural numbers. Note that $\mu_{\varepsilon}(N_i) < \varepsilon$ and that $\{N_i\}_{i \in \mathbb{N}}$ is a monotone increasing sequence of finite sets such that $\lim_{i \rightarrow \infty}N_i = \mathbb{N}$. Therefore, we have that $1 = \mu_{\varepsilon}(\mathbb{N}) = \mu_{\varepsilon}(\lim_{i \rightarrow \infty}N_i) = \lim_{i \rightarrow \infty}\mu_{\varepsilon}(N_i) < \varepsilon$, a contradiction. Accordingly, for any $\varepsilon \in [0, 1]$, no matter what measure $\hat{\mu}_t$ the algorithm picks to make its prediction in round $t$, there always exists a finite set $A_{\hat{\mu}_t}$ such that $\hat{\mu}_t(A_{\hat{\mu}_t}) \geq \varepsilon$. Since $|A_{\hat{\mu}_t}| < \infty$, we know that $A_{\hat{\mu}_t}^c \in \mathcal{S}(\mathcal{Y})$. Thus, there is always a strategy for the adversary to force the learner's expected loss to be at least $\varepsilon$ in each round $t \in [T]$. On the other hand, since for any sequence of sets $S_1, ..., S_T \in \mathcal{S}(\mathcal{Y})$, we have that $\cap_{t=1}^T S_t \neq \emptyset$, there exists a hypothesis $h_y \in \mathcal{H}$ such that $h_y(x) \in S_t$ for all $x \in \mathcal{X}$ and $t \in [T]$. Thus, every stream is realizable by $\mathcal{H}$. Accordingly, for every $\varepsilon \in [0, 1]$, the expected regret of any online learner in the realizable setting is at least $\varepsilon T$. 
 \end{example}
 
Example \ref{exm:counter} shows that, in full generality, one might need to go beyond trees with finite width in order to characterize online learnability with set-valued feedback. Using this observation, we define two new combinatorial dimensions, the Set Littlestone and Measure shattering dimension, whose associated trees can have infinite-width. In Section \ref{sec:real}, we show that the Set Littlestone dimension (SLdim) tightly characterizes the online learnability of $\mathcal{H}$ by any \textit{deterministic} online learner in the \textit{realizable} setting.

\begin{definition}[Set Littlestone dimension]\label{def:sldim}
\noindent Let $\mathcal{T}$ be a complete $\mathcal{X}$-valued, $\mathcal{Y}$-ary tree of depth $d$. The tree $\mathcal{T}$ is shattered by $\mathcal{H} \subseteq \Ycal^{\Xcal}$  if there exists a sequence $(f_1, ..., f_d)$ of edge-labeling set-valued functions  $f_t: \mathcal{Y}^{t} \rightarrow \mathcal{S}(\mathcal{Y})$  such that for every path $y = (y_1, ..., y_d) \in \mathcal{Y}^d$, we have $y_t \notin f_t(y_{\leq t})$ and there exists a hypothesis $h_{y} \in \mathcal{H}$ such that  $h_{y}(\mathcal{T}_t(y_{<t})) \in f_t(y_{\leq t})$ for all $t \in [d]$. The Set Littlestone dimension of $\mathcal{H}$, denoted  $\emph{\texttt{SL}}(\mathcal{H}, \mathcal{S}(\mathcal{Y}))$, is the maximal depth of a tree $\mathcal{T}$ that is shattered by $\mathcal{H}$. If there exists shattered trees of arbitrarily large depth, we say $\emph{\texttt{SL}}(\mathcal{H}, \mathcal{S}(\mathcal{Y})) = \infty$.
\end{definition}

 On the other hand, we show that the Measure Shattering dimension  (MSdim) characterizes the online learnability of $\mathcal{H}$ by any \textit{randomized} online learner in both the realizable and agnostic settings under set-valued feedback. We note that the Measure Shattering dimension is similar to the sequential fat-shattering dimension in the sense that it is a \textit{scale-sensitive}, and therefore defined at every $\gamma > 0$. 

\begin{definition}[Measure Shattering dimension]\label{def:msdim}
\noindent Let $\mathcal{T}$ be a complete $\mathcal{X}$-valued, $\Pi(\mathcal{Y})$-ary tree of depth $d$, and fix $\gamma \in (0,1] $. The tree $\mathcal{T}$ is $\gamma$-shattered by $\mathcal{H} \subseteq \Ycal^{\Xcal}$  if there exists a sequence $(f_1, ..., f_d)$ of edge-labeling set-valued functions  $f_t: \Pi(\mathcal{Y})^{t} \rightarrow \mathcal{S}(\mathcal{Y})$  such that for every path $\mu = (\mu_1, ..., \mu_d) \in \Pi(\mathcal{Y})^d$, we have $\mu_t(f_t(\mu_{\leq t})) \leq 1 - \gamma$ and there exists a hypothesis $h_{\mu} \in \mathcal{H}$ such that  $h_{\mu}(\mathcal{T}_t(\mu_{<t})) \in f_t(\mu_{\leq t})$  for all $t \in [d]$. The Measure Shattering dimension of $\mathcal{H}$ at scale $\gamma$, denoted $\emph{\texttt{MS}}_{\gamma}(\mathcal{H}, \mathcal{S}(\mathcal{Y}))$, is the maximal depth of a tree $\mathcal{T}$ that is $\gamma$-shattered by $\mathcal{H}$. If there exists $\gamma$-shattered trees of arbitrarily large depth, we say $\emph{\texttt{MS}}_{\gamma}(\mathcal{H}, \mathcal{S}(\mathcal{Y})) = \infty$. Analogously, we can define $ \emph{\texttt{MS}}_{0}(\mathcal{H}, \mathcal{S}(\mathcal{Y}))$ by requiring strict inequality, $ \mu_t(f_t(\mu_{\leq t})) <1 $.
\end{definition}

 As with most scale-sensitive dimensions, MSdim has a monotonicity property, namely, $\texttt{MS}_{\gamma_1}(\mathcal{H}) \leq \texttt{MS}_{\gamma_2}(\mathcal{H}) $ for any $\gamma_2 \leq \gamma_1$. This follows immediately upon noting that for any $A \in \mathcal{S}(\mathcal{Y})$, we have $\mu(A) \leq 1- \gamma_1 \leq 1-\gamma_2$. Thus, a tree shattered at scale $\gamma_1$ is also shattered at scale $\gamma_2$.


\subsection{Relations Between Combinatorial Dimensions}


In this section, we show how the $\text{$p$-SLdim}$, $\text{SLdim}$, and $\text{MSdim}$ are related under various conditions on the problem setting. One natural case is when the set system $\mathcal{S}(\mathcal{Y})$ has finite \textit{Helly number}, a quantification of the following property: every collection-wise disjoint sequence of sets in $\mathcal{S}(\mathcal{Y})$ contains a \textit{small} collection-wise disjoint subsequence of sets.  

\begin{definition}[Helly Number of $\mathcal{S}(\mathcal{Y})$]{}
\noindent  The Helly number of $\mathcal{S}(\mathcal{Y}) \subseteq \sigma(\mathcal{Y})$, denoted $\emph{\texttt{H}}(\mathcal{S}(\mathcal{Y}))$, is the smallest number $p \in \mathbbm{N}$ such that for any collection of sets $\mathcal{C} \subseteq \mathcal{S}(\mathcal{Y})$ where $\bigcap_{S \in \mathcal{C}} S = \emptyset$, there is a subset $\mathcal{C}^{\prime} \subset \mathcal{C}$ of size at most $p$ where $\bigcap_{S \in \mathcal{C}^{\prime}} S = \emptyset$.
\end{definition}

We say that $\mathcal{S}(\mathcal{Y})$ is a Helly space if and only if $\texttt{H}(\mathcal{S}(\mathcal{Y})) < \infty$. The Helly property captures many practical learning settings. For example, when $\mathcal{Y}$ is finite, any collection $\mathcal{S}(\mathcal{Y}) \subseteq \sigma(\mathcal{Y})$ is a Helly space. However, Helly spaces are more general and capture situations where $\mathcal{Y}$ can be uncountably large. For example, if $\mathcal{Y} = [0, 1]$ and $\mathcal{S}(\mathcal{Y}) = \{[a, b]: 0 \leq a < b \leq 1\}$ is the set of all intervals in $\mathcal{Y}$, then the celebrated Helly's theorem states that $\texttt{H}(\mathcal{S}(\mathcal{Y})) = 2$ \citep{radon1921mengen}. In Section \ref{sec:app}, we give even more examples of natural settings where $\texttt{H}(\mathcal{S}(\mathcal{Y})) < \infty$. In this work, we use the Helly number of $\mathcal{S}(\mathcal{Y})$ to establish a relationship between the combinatorial dimensions defined above.

\begin{theorem}[Structural Properties]\label{thm:relation}
\noindent For $\mathcal{S}(\mathcal{Y}) \subseteq \sigma(\Ycal)$ and $\mathcal{H} \subseteq \mathcal{Y}^{\mathcal{X}}$, we have 

\begin{itemize}
    \item[(i)] $ \emph{\texttt{SL}}_p(\Hcal) \leq \emph{\texttt{MS}}_{\gamma}(\Hcal) \leq  \emph{\texttt{SL}}(\Hcal)$ for all $p\geq 2$ and $\gamma \in [0,\frac{1}{p}]$.  
     \item[(ii)] If  $p = \emph{\texttt{H}}(\Scal(\Ycal)) < \infty$, then  $\emph{\texttt{SL}}_p(\Hcal) =  \emph{\texttt{MS}}_{\gamma}(\Hcal) =  \emph{\texttt{SL}}(\Hcal)$ for all $\gamma \in [0,\frac{1}{p}]$.
\end{itemize}
\end{theorem}

 The proof of Theorem \ref{thm:relation} is found in Appendix \ref{app:prfrelation}. The key idea in the proof of (ii) is that when $\mathcal{S}(\mathcal{Y})$ is a Helly space, we can ``compress"  the infinite-width trees in the definition of SLdim and MSdim to finite-width trees used in the definition of $p$-$\text{SLdim}$.  Perhaps the most important implication of these relations is that when $\mathcal{S}(\mathcal{Y})$ is a Helly family, deterministic and randomized realizable learnability are equivalent and characterized by the same dimension. Thus, as we show in Section \ref{sec:separation}, the separation between randomized and deterministic realizable learnability only occurs when $\texttt{H}(\mathcal{S}(\mathcal{Y})) = \infty$. We leave it as an open question whether the finiteness of $\texttt{H}(\mathcal{S}(\mathcal{Y}))$ is necessary for this equivalence.

\section{Realizable Setting} \label{sec:real}

\subsection{A Separation Between Deterministic and Randomized Learnability}\label{sec:separation}
We first show that unlike in online multiclass learning with single-label feedback, deterministic and randomized learnability are not equivalent under set-valued feedback. We note that \cite{hanneke2023bandit, hanneke2021online} show a similar separation in the context of bandit learnability and proper online learnability. 

\begin{theorem}[Deterministic Learnability $\not \equiv$  Randomized Learnability]{}{}
\noindent There exists a $\mathcal{Y}$, $\mathcal{S}(\mathcal{Y})$ and $\mathcal{H} \subseteq \mathcal{Y}^{\mathcal{X}}$ such that in the realizable setting (i) $\mathcal{H}$ is online learnable, however (ii) no deterministic algorithm is an online learner for $\mathcal{H}$.
\end{theorem}

\begin{proof}
Let $\mathcal{Y} = \mathbbm{N}$ and $\mathcal{S}(\mathcal{Y}) = \{A_y\}_{y \in \mathcal{Y}}$ where $A_y = \mathbbm{N} \setminus y$. Let $\mathcal{H} = \{h_y: y \in \mathbbm{N}\}$ be the set of constant functions. That is, $h_y(x) = y$ for all $x \in \mathcal{X}$. 

Let $\mathcal{A}$ be any deterministic online learner for $\mathcal{H}$ and $T \in \mathbb{N}$ be the time horizon. We  construct a realizable stream of length $T$ such that $\mathcal{A}$ makes a mistake on each round. Without loss of generality, we let the adversary play after $\mathcal{A}$ since $\mathcal{A}$ is deterministic. To that end, pick any sequence of instances $\{x_t\}_{t=1}^T \in \mathcal{X}^T$ and consider the labeled stream $\{(x_t, A_{\mathcal{A}(x_t)})\}^T_{t=1}$, where $\mathcal{A}(x_t)$ denotes the prediction of $\mathcal{A}$ in the $t$'th round. By definition of $A_y$, we have $\sum_{t=1}^T \mathbbm{1}\{\mathcal{A}(x_t) \notin A_{\mathcal{A}(x_t)}\} = T$. Moreover, since $T$ is finite, it also holds that $\bigcap_{t=1}^T A_{\mathcal{A}(x_t)} \neq \emptyset$. Thus, there exists $h_y \in \mathcal{H}$ such that for all $t \in [T]$, $h_y(x_t) \in A_{\mathcal{A}(x_t)}$, showing that the stream  $\{(x_t, A_{\mathcal{A}(x_t)})\}^T_{t=1}$ is indeed realizable. Since $\mathcal{A}$ is arbitrary, every deterministic algorithm fails to learn $\mathcal{H}$ under set-valued feedback from $\mathcal{S}(\mathcal{Y})$.

We now give a randomized online learner for $\mathcal{H}$ that achieves sub-linear regret for any sequence of instances labeled by sets from $\mathcal{S}(\mathcal{Y})$. Let $\{(x_t, S_t)\}_{t=1}^T \in (\mathcal{X} \times \mathcal{S}(\mathcal{Y}))^T$ denote the stream of instances to be observed by the randomized online learner. Consider a randomized learner $\mathcal{A}$ that in each round samples uniformly from $\{1, ..., T\}$. Then, $\mathcal{A}$'s expected cumulative loss satisfies
$$\mathbbm{E}\left[\sum_{t=1}^T \mathbbm{1}\{\mathcal{A}(x_t) \notin S_t\} \right] = \sum_{t=1}^T \mathbbm{P}\left[\mathcal{A}(x_t) \notin S_t \right] = \sum_{t=1}^T \mathbbm{P}\left[S_t = A_{\mathcal{A}(x_t)} \right] \leq \sum_{t=1}^T \frac{1}{T} = 1,$$
where we have used the fact that $\mathcal{A}(x_t) \notin S_t$ iff the adversary exactly picks the set $S_t = A_{\mathcal{A}(x_t)}$. Thus, $\mathcal{A}$ achieves a constant regret bound, showcasing that it is an online learner for $\mathcal{H}$ under set-valued feedback from $\mathcal{S}(\mathcal{Y})$. This completes the overall proof as we have given a learning setting that is online learnable, but not by any deterministic learner.  
\end{proof}

\subsection{Deterministic Learnability} \label{sec:detlearn}

Given that deterministic and randomized online learnability are not generally equivalent, we show that the SLdim tightly characterizes \emph{deterministic} online learnability in the realizable setting. 

\begin{theorem}[Deterministic Realizable Learnability]\label{thm:det_real}
For any $\mathcal{S}(\mathcal{Y}) \subseteq \sigma(\mathcal{Y})$ and $\mathcal{H} \subseteq \mathcal{Y}^{\mathcal{X}}$, we have $\inf_{\emph{\text{Deterministic }} \mathcal{A}} \emph{\texttt{M}}_{\mathcal{A}}(T,\mathcal{H}) = \emph{\texttt{SL}}(\mathcal{H}).$
\end{theorem}

Our proof of the upperbound on the optimal $ \texttt{M}_{\mathcal{A}}(T,\mathcal{H})$ is constructive. We show that Algorithm \ref{alg:det_SOA} makes at most $\texttt{SL}(\mathcal{H})$ mistakes in any realizable stream by generalizing the arguments by \cite{Littlestone1987LearningQW}. To prove the lowerbound on $\texttt{M}_{\mathcal{A}}(T,\mathcal{H})$ for any deterministic algorithm $\mathcal{A}$, we construct a difficult stream by traversing the shattered tree of depth $\texttt{SL}(\mathcal{H})$ adapting to $\mathcal{A}$'s predictions. Both proofs can be found in Appendix \ref{app:det_real}.

\begin{algorithm}
\setcounter{AlgoLine}{0}
\caption{Deterministic Standard Optimal Algorithm}\label{alg:det_SOA}
Initialize $V_{0} = \mathcal{H}$

\For{$t = 1,...,T$} {

    Receive unlabeled example $x_t \in \mathcal{X}$.

    For each $A \in \Scal(\Ycal)$, define $V_{t-1}(A) := \{h \in V_{t-1} \, \mid \, h(x_t) \in A\}$. 

    Let $\mathcal{S}_t(\mathcal{Y}) := \{A \in \mathcal{S}(\mathcal{Y}): A\cap \{h(x_t) \, \mid \, h \in V_{t-1}\} \neq \emptyset\}.$

    If $\texttt{SL}(V_{t-1}) > 0$, predict 
    $\hat{y}_t = \argmin_{y \in \Ycal} \,\max_{\substack{A \in \Scal(\Ycal) \\ y \notin A}} \, \texttt{SL}(V_{t-1}(A)).$
    
    Else, predict $\hat{y}_t \in \bigcap_{A \in \mathcal{S}_t(\mathcal{Y})}\,  A$. 
    
    Receive feedback $S_t\in \Scal_t(\Ycal)$ and update $V_t = V_{t-1}(S_t).$
}
\end{algorithm}







\noindent \textbf{Remark.} We highlight that Algorithm \ref{alg:det_SOA} generalizes the classical Standard Optimal Algorithm. In fact, if $\mathcal{S}(\mathcal{Y}) = \{\{y\}: y \in \mathcal{Y}\}$ then Algorithm \ref{alg:det_SOA} reduces exactly to the classical Standard Optimal Algorithm from \cite{Littlestone1987LearningQW} and SLdim reduces to the Ldim. Moreover, when  $\mathcal{S}(\mathcal{Y}) = \{\mathcal{Y} \setminus \{y\}: y \in \mathcal{Y}\}$, Algorithm \ref{alg:det_SOA} reduces to the Bandit Standard Optimal Algorithm from \cite{DanielyERMprinciple} and SLdim reduces to the Bandit Littlestone dimension. 

\subsection{Randomized Learnability} \label{sec:rand_real}

Next, we characterize randomized online learnability in the realizable setting. The proof of Theorem \ref{thm:rand_real} can be found in Appendix \ref{app:rand_learn}.

\begin{theorem}[Randomized Realizable Learnability]\label{thm:rand_real}
For any $\mathcal{S}(\mathcal{Y}) \subseteq \sigma(\mathcal{Y})$ and $\mathcal{H} \subseteq \mathcal{Y}^{\mathcal{X}}$, 
$$\sup_{\gamma \in (0, 1] } \, \gamma\,  \emph{\texttt{MS}}_{\gamma}(\mathcal{H}) \leq \inf_{\mathcal{A}} \, \emph{\texttt{M}}_{\mathcal{A}}(T,\mathcal{H}) \leq C \inf_{\gamma \in (0,1]}\Big\{\gamma T + \int_{\gamma}^1 \emph{\texttt{MS}}_{\eta}(\mathcal{H}) d\eta\Big\}$$
\noindent where $C > 0$ is some universal constant. Moreover, both the upper and lowerbounds can be tight in general up to constant factors.





    



\end{theorem}

Using Theorem \ref{thm:relation}, it follows that $\texttt{M}_{\mathcal{A}}(T, \Hcal) = \Theta(\texttt{SL}(\Hcal))$ whenever $\texttt{H}(\mathcal{S}(\Ycal)) < \infty$. We highlight that the upperbound can be tight up to logarithmic factors in $T$. If $\mathcal{S}(\mathcal{Y})$ is a set of singletons, then we have $\texttt{MS}_0(\mathcal{H}) = \texttt{L}(\mathcal{H})$. Thus, the upperbound reduces to $\texttt{L}(\mathcal{H})$, which matches the known lowerbound of $\texttt{L}(\mathcal{H})/2$ in the realizable multiclass classification \citep{DanielyERMprinciple}.
 Example \ref{exm: agnlbtight} shows that the lowerbound of $\sup_{\gamma  > 0 } \, \gamma\,  \texttt{MS}_{\gamma}(\mathcal{H})$ can be tight in the realizable setting. 

To achieve our upperbound, we first construct a randomized online learner running at a fixed scale $\gamma \in (0, 1)$, whose expected cumulative loss, in the realizable setting, is at most $\gamma T + \texttt{MS}_{\gamma}(\mathcal{H})$. Then, we upgrade this result by adapting the algorithmic chaining technique from \cite{daskalakis2022fast} to give a randomized, \textit{multi-scale} online learner in the realizable setting. Our lowerbound is obtained by traversing the tree of depth $\texttt{MS}_{\gamma}(\mathcal{H})$ adapting to the distributions that the algorithm produces to make its randomized predictions.


We conclude this section by showing that the Helly number of $\mathcal{S}(\mathcal{Y})$ is a sufficient condition for deterministic and randomized learnability to be equivalent in the realizable setting. Corollary \ref{cor:det_ran_equiv} follows directly upon using Theorems \ref{thm:relation}(ii), \ref{thm:det_real},  and \ref{thm:rand_real}. 

\begin{corollary}[Deterministic Learnability $ \equiv$  Randomized Learnability for Helly Families]\label{cor:det_ran_equiv}
\noindent Let $\mathcal{S}(\mathcal{Y}) \subseteq \sigma(\mathcal{Y})$ such that $ \emph{\texttt{H}}(\mathcal{S}(\mathcal{Y})) < \infty$. Then, in the realizable setting, $\Hcal \subseteq \Ycal^{\Xcal}$ is online learnable via a randomized algorithm if and only if $\Hcal$ is online learnable via a deterministic algorithm. 
\end{corollary}



%
\section{Agnostic Setting} \label{sec:agn}

In this section, we move beyond the realizable setting, and consider the more general agnostic setting, where we are not guaranteed that there exists a consistent hypothesis. Our main theorem shows that the finiteness of MSdim at every scale $\gamma > 0$ is both a necessary and sufficient condition for agnostic online learnability with set-valued feedback.


\begin{theorem}[Agnostic Learnability]\label{thm:agn} For any $\mathcal{S}(\mathcal{Y}) \subseteq \sigma(\mathcal{Y})$ and  hypothesis class $\mathcal{H} \subseteq \mathcal{Y}^{\mathcal{X}}$ where $\sup_{\gamma \in (0, 1] } \emph{\texttt{MS}}_{\gamma}(\mathcal{H}) > 0$, 
$$\max \left\{\sqrt{\frac{\emph{\texttt{SL}}_2(\mathcal{H})\,T}{8}}, \sup_{\gamma \in (0, 1] } \, \gamma\,  \emph{\texttt{MS}}_{\gamma}(\mathcal{H}) \right\} \leq \inf_{\mathcal{A}}\emph{\texttt{R}}_{\mathcal{A}}(T,\mathcal{H}) \leq \inf_{\gamma  \in (0, 1] }\, \left\{\emph{\texttt{MS}}_{\gamma}(\mathcal{H}) + \gamma T + \sqrt{2\,\emph{\texttt{MS}}_{\gamma} (\mathcal{H})\,T \ln(T)}\right\}$$
\noindent and the upper and lowerbounds can be tight in general up to constant factors. Moreover, when $\sup_{\gamma \in (0, 1] } \emph{\texttt{MS}}_{\gamma}(\mathcal{H}) = 0$, there is no non-negative lowerbound.



    




\end{theorem}

Using Theorem \ref{thm:relation}, it follows that $\texttt{R}_{\mathcal{A}}(T, \Hcal) = \tilde{\Theta}(\sqrt{T})$ whenever $ \texttt{H}(\mathcal{S}(\Ycal)) < \infty$ and $\texttt{SL}(\Hcal)<\infty$. We highlight that the upper bound can be tight up to logarithmic factors in $T$. If $\mathcal{S}(\mathcal{Y})$ is a set of singletons, then we have $\texttt{MS}_0(\mathcal{H}) = \texttt{L}(\mathcal{H})$. Thus, the upper bound reduces to $\texttt{L}(\mathcal{H}) + \sqrt{2\, \texttt{L}(\mathcal{H})\, T\ln(T)}$, which matches the known lower bound of $\sqrt{\texttt{L}(\mathcal{H})\,T/8}$ in the agnostic multiclass classification \citep{DanielyERMprinciple}. The following example shows that the lower bound cannot be improved in general. 



 \begin{example} \label{exm: agnlbtight}
 \noindent Let $\mathcal{Y} = \{1,2,3,4,5,6\}$, $\mathcal{S}(\mathcal{Y}) = \{\{1, 4,5\}, \{2, 5,6\}, \{3, 4,6\}\}$, and $\mathcal{H} = \{h_1,h_2, h_3\}$, where again $h_i$ is the hypothesis that always outputs $i$. Let $d = \emph{\texttt{SL}}_2(\mathcal{H})$ and $d_{\gamma} = \emph{\texttt{MS}}_{\gamma}(\mathcal{H})$. Since there are no disjoint sets in $\Scal(\Ycal)$, we trivially have $d=0$, reducing the lowerbound to $\gamma\,d_{\gamma}$. First, we prove that $\sup_{\gamma} \gamma d_{\gamma} = \frac{1}{3}$. This follows from the fact that $\emph{\texttt{H}}(\mathcal{S}(\mathcal{Y})) = 3$, and therefore, by Theorem \ref{thm:relation}, for all $\gamma \in [0, \frac{1}{3}]$ we have $d_{\gamma} = \emph{\texttt{SL}}(\mathcal{H}) = 1$. Moreover, by the monotonicty property of MSdim,   $d_{\gamma} \leq d_{\frac{1}{3}} = 1$ for all $\gamma > \frac{1}{3}$. Thus, it must be the case $\sup_{\gamma > 0} \gamma\, d_{\gamma} = \frac{1}{3}$. 
 
 Now, we give a randomized online learner whose expected regret is at most $\sup_{\gamma > 0} \gamma d_{\gamma} = \frac{1}{3}$ on the worst-case sequence, matching the lowerbound. Consider an online learner $\mathcal{A}$, which on the round $t = 1$ predicts by uniformly sampling from $\{4,5,6\}$, and on all other rounds predicts by uniformly sampling from $\{4,5,6\} \cap S_{t-1}$, where $S_{t-1}$ is the set revealed by the adversary on round $t - 1$.  Our goal will be to show that $\mathcal{A}$'s expected regret on any sequence is at most $\frac{1}{3}$. Let $\{(x_t, S_t)\}_{t=1}^T$ denote the stream chosen by the adversary. Then, we have
\begin{align*}
\mathbbm{E}\left[\sum_{t=1}^T\mathbbm{1}\{\mathcal{A}(x_t) \notin S_t\} \right] 
    &= \frac{1}{3} +  \sum_{t=2}^T  \mathbbm{E}\left[\mathbbm{1}\{\mathcal{A}(x_t) \notin S_t\}|S_t \neq S_{t-1} \right]\mathbbm{1}\{S_t \neq S_{t-1} \} = \frac{1}{3} +  \frac{1}{2}\sum_{t=2}^T \mathbbm{1}\{S_t \neq S_{t-1} \}, 
\end{align*}
where the first equality follows from the fact that $\mathbbm{E}\left[\mathbbm{1}\{\mathcal{A}(x_1) \notin S_1\} \right] = \frac{1}{3}$ and  $\mathbbm{E}\left[\mathbbm{1}\{\mathcal{A}(x_t) \notin S_t\}\, |\, S_t = S_{t-1} \right] = 0$. Moreover, we can lowerbound the expected cumulative loss of the best fixed hypothesis as 
$$\min_{h \in \mathcal{H}}\sum_{t=1}^T\mathbbm{1}\{h(x_t) \notin S_t\}  = \min_{i \in [3]}\sum_{t=1}^T\mathbbm{1}\{i \notin S_t\} 
    \geq \frac{1}{2}\sum_{t=2}^T\mathbbm{1}\{S_t \neq S_{t-1}\} $$
Combining the upper- and lowerbound gives that  $\emph{\texttt{R}}_{\mathcal{A}}(T, \mathcal{H}) \leq \frac{1}{3}$.
 \end{example}

\noindent \textbf{Remark.} An important implication of Theorem \ref{thm:agn} is that when $\texttt{H}(\mathcal{S}(\mathcal{Y})) = 2$, a lowerbound scaling with $T$ is always possible. However, Example \ref{exm: agnlbtight} above witnessing the tightness of the lowerbounds in Theorem \ref{thm:agn} shows that this is not the case when $\texttt{H}(\mathcal{S}(\mathcal{Y})) \geq 3$. Thus, a sharp phase transition occurs when  $\texttt{H}(\mathcal{S}(\mathcal{Y}))$ increases from $2$ to $3$. 
\vspace{5pt}

\section{Applications} \label{sec:app}
In this section, we show how online multilabel ranking with relevance-score feedback and online multilabel classification are special instances of our model of online learning with set-valued feedback.  In Appendix \ref{appdx:applications}, we also consider real-valued prediction with interval-valued response. 
\subsection{Online Multilabel Ranking}
In online multilabel ranking, we let $\mathcal{X}$ denote the instance space, $\mathcal{Y}$ denote the set of permutations over labels $[K]:= \{1, ..., K\}$, and $\mathcal{R} = \{0, 1\}^K$ denote the target space for some $K \in \mathbbm{N}$.  We refer to an element $r \in \mathcal{R}$ as a \textit{binary relevance-score vector} that indicates the relevance of each of the $K$ labels. A permutation $\pi \in \mathcal{Y}$  induces a \textit{ranking} of the $K$ labels in decreasing order of relevance. For an index $i \in [K]$, we let $\pi^i \in [K]$ denote the \textit{rank} of label $i$. Likewise, given an index $i \in [K]$, we let $r^i$ denote the relevance of label $i$. A ranking hypothesis $h \in \mathcal{H} \subseteq \mathcal{Y}^{\mathcal{X}}$ maps instances in $\mathcal{X}$ to a permutation (ranking) in $\mathcal{Y}$. Given an instance $x \in \mathcal{X}$, one can think of $h(x)$ as $h$'s ranking of the $K$ different labels in decreasing order of relevance. 

Unlike classification, a distinguishing property of multilabel ranking is the \textit{mismatch} between the predictions the learner makes and the feedback it receives.
Because of this mismatch, there is no canonical loss in multilabel ranking like the 0-1 loss in classification. Nevertheless, a natural analog of the 0-1 loss in multilabel ranking is
 $\ell_{0\text{-}1}(\pi, r) = \sup_{i, j \in [K]}\mathbbm{1}\{r^i < r^j\}\, \mathbbm{1}\{\pi^i < \pi^j\}.$ At a high-level, the 0-1 ranking loss penalizes a permutation $\pi$ if it ranks a less relevant item above a more relevant item. 
 
 Under the 0-1 loss,  online multilabel ranking with binary relevance-score feedback is a specific instance of our general online learning model with set-valued feedback.  To see this, note that given a relevance score vector $r \in \mathcal{R}$, there can be many permutations $\pi \in \mathcal{Y}$ such that $\ell_{0\text{-}1}(\pi, r) = 0$. Indeed, suppose $r = (0, 1, 1)$. Then,  both the permutations $\pi_1 = (3, 1, 2)$ and $\pi_2 = (3, 2, 1)$ achieve $0$ loss. Thus, an \textit{equivalent} way of representing $r = (0, 1, 1)$ is to consider the set of permutations in $\mathcal{Y}$ for which $\ell_{0\text{-}1}(\pi, r) = 0$. To this end, given any $r \in \mathcal{R}$, let $\mathcal{Y}(r) = \{\pi \in \mathcal{Y}: \ell_{0\text{-}1}(\pi, r) = 0\}$. Then, note that for every $\pi \in \mathcal{Y}$ and $r \in \mathcal{R}$, we have $\ell_{0\text{-}1}(\pi, r) = \mathbbm{1}\{\pi \notin \mathcal{Y}(r)\} $. From this perspective, we can equivalently define the online multilabel ranking setting by having the adversary in each round $t \in [T]$, reveal a \textit{set} $\mathcal{Y}(r_t) \in \{\mathcal{Y}(r): r \in \mathcal{R}\} = \mathcal{S}(\mathcal{Y})$ instead of the binary relevance score vector $r_t \in \mathcal{R}$, and penalizing the learner according to the 0-1 \textit{set loss} $\mathbbm{1}\{\pi_t \notin \mathcal{Y}(r_t)\}$, instead of  $\ell_{0\text{-}1}(\pi, r)$.

Since online multilabel ranking is a specific instance of our general online learning with set-valued feedback, our qualitative characterization in terms of the SLdim and MSdim carry over. 
Thus, in this section, we instead focus on establishing a sharp quantitative characterization of online learnability. To do so, we first show that $\texttt{H}(\mathcal{S}(\mathcal{Y})) = 2$. The proof of Lemma \ref{lem:helly_rank} is deferred to Appendix \ref{appdx:mlr}.

\begin{lemma}[Helly Number of Permutation Sets]\label{lem:helly_rank}
\noindent Let $\mathcal{S}(\mathcal{Y}) = \{\mathcal{Y}(r): r \in \mathcal{R}\}$ where $\mathcal{Y}(r) = \{\pi \in \mathcal{Y}:\ell_{0\text{-}1}(\pi, r) = 0\}$. Then, $\emph{\texttt{H}}(\mathcal{S}(\mathcal{Y})) = 2$. 
\end{lemma}

Since $\texttt{H}(\mathcal{S}(\mathcal{Y})) = 2$, by Theorem \ref{thm:relation}, we know that for all $\gamma \in [0, \frac{1}{2}]$, $\texttt{SL}_2(\mathcal{H}) = \texttt{MS}_{\gamma}(\mathcal{H}) = \texttt{SL}(\mathcal{H})$.  Therefore, the $\texttt{SL}_2(\mathcal{H})$ characterizes both deterministic and randomized online multilabel ranking learnability. Moreover, we can use Theorems \ref{thm:det_real}, \ref{thm:rand_real}, and \ref{thm:agn} to give Corollary \ref{cor:mlrquant}, a sharp quantitative characterization of online multilabel ranking learnability in both the realizable and agnostic settings. 

\begin{corollary}[Online Learnability of Multilabel Ranking]\label{cor:mlrquant}
\noindent Let $\mathcal{Y}$, $\mathcal{R}$, and $\mathcal{S}(\mathcal{Y})$ be defined as above. For any ranking hypothesis class $\mathcal{H} \subseteq \mathcal{Y}^{\mathcal{X}}$ we have 
\begin{itemize}
    \item[(i)] $\frac{\emph{\texttt{SL}}_2(\mathcal{H})}{2} \leq \inf_{\mathcal{A}} \, \emph{\texttt{M}}_{\mathcal{A}}(T,\mathcal{H}) \leq \emph{\texttt{SL}}_2(\mathcal{H})$.
    \item[(ii)] $\sqrt{\frac{\emph{\texttt{SL}}_2(\mathcal{H})\, T}{8}} \leq \inf_{\mathcal{A}} \, \emph{\texttt{R}}_{\mathcal{A}}(T,\mathcal{H}) \leq \emph{\texttt{SL}}_2(\mathcal{H}) + \sqrt{2\,\emph{\texttt{SL}}_2(\mathcal{H})\,T \ln(T)}$.
\end{itemize}

\end{corollary}
 
\noindent We note that the infimum in Corollary \ref{cor:mlrquant}(i) is over all algorithms, not just deterministic ones. Also, observe that the upper- and lowerbounds in Corollary \ref{cor:mlrquant} do not depend on $|\mathcal{Y}|$ or $|\mathcal{R}|.$




\subsection{Online Multilabel Classification}
In online multilabel \textit{classification}, we let $\mathcal{X}$ denote the instance space, and $\mathcal{Y} = \{0, 1\}^K$ is the set of all bit strings of length $K \in \mathbbm{N}$. Unlike multilabel ranking, instead of predicting a permutation over $[K]$, the goal is to predict $\hat{y} \in \mathcal{Y}$, which indicates which of the labels are relevant. As feedback, the learner also receives a bit string $y \in \mathcal{Y}$ which gives the ground truth on which of the $K$ labels are relevant. A multilabel hypothesis $h \in \mathcal{H} \subseteq \mathcal{Y}^{\mathcal{X}}$ maps instances in $\mathcal{X}$ to a bit string in $\mathcal{Y}$.

The most natural loss in multilabel classification is the Hamming loss, defined by
$\ell_{H}(\hat{y}, y) = \sum_{i=1}^K \mathbbm{1}\{\hat{y}^i \neq y^i\}$. However, when $K$ is very large, evaluating performance using the Hamming loss might be too stringent. Instead, it might be more natural to consider a  thresholded version of the Hamming loss, defined as $\ell_{H, q}(\hat{y}, y) = \mathbbm{1}\{\ell_H(\hat{y}, y) > q\} = \mathbbm{1}\{\hat{y} \notin \mathcal{B}(y, q)\},$ where $q \in [K-1]$ and $\mathcal{B}(y, q) = \{\hat{y} \in \mathcal{Y}: \ell_{H}(\hat{y}, y) \leq q\}$ denotes the Hamming ball of radius $q$ centered at $y$. The loss $\ell_{H, q}$ allows the learner's prediction $\hat{y}$ to be incorrect in at most $q$ different spots before penalizing the learner. By taking $\mathcal{Y} = \{0, 1\}^K$ and $\mathcal{S}_q(\mathcal{Y}) = \{\mathcal{B}(y, q): y \in \mathcal{Y}\}$, it is not hard to see that online multilabel classification with $\ell_{H, q}$ is a specific instance of our general online learning model with set-valued feedback. Thus, a quantitative characterization of online multilabel classification in terms of $\texttt{SL}(\Hcal)$ and $\texttt{MS}_{\gamma}(\Hcal)$ follows immediately from Theorems \ref{thm:det_real} and \ref{thm:agn}. The precise statement is provided in Appendix \ref{appdx:mlc}. 





In multilabel ranking, we showed that the $\text{2-SLdim}$, provides a tight quantitative characterization of online learnability without any dependence on $K$. Such a characterization in terms of the $2$-SLdim, as opposed to SLdim or MSdim, is desirable because it satisfies the Finite Character Property \cite[Definition 4]{ben2019learnability}.  A crucial step in doing so was showing that the Helly number of the permutation set system is $2$, and more importantly, does not scale with $K$. Along this direction, it is natural to ask whether there exists a $p \in \mathbbm{N}$ such that the $\text{$p$-SLdim}$ gives a $K$-free quantitative characterization of online multilabel classification under $\ell_{H, q}$. To resolve this question positively it suffices to show that  $\texttt{H}(\mathcal{S}_q(\mathcal{Y}))$ does not scale with $K$, as conjectured below. 


\begin{conjecture}[Helly Number of Hamming Balls]
\noindent For any $K \in \mathbb{N}$ and $q \in [K-1]$, we have that $\emph{\texttt{H}}(\mathcal{S}_{q}(\mathcal{Y})) = 2^{q+1}$.
\end{conjecture}

In Appendix \ref{appdx:mlc}, we partially resolve this conjecture by showing $2^{q+1} \leq \texttt{H}(\mathcal{S}_q(\mathcal{Y})) \leq \sum_{r=0}^q \binom{K}{r} + 1$. We leave it as an open question to prove a matching upperbound. 

\vspace{5pt}

\noindent \textbf{Remark.} This conjecture has been recently resolved by \cite{alon2024helly}, who prove a matching upperbound of $2^{q+1}$.

\section*{Acknowledgements}
AT acknowledges the support of NSF via grant IIS-2007055. VR acknowledges the support of the NSF Graduate Research Fellowship. US acknowledges the support of the Rackham International Student Fellowship.

\bibliographystyle{plainnat}
\bibliography{references}

\newpage
\appendix

\section{Relationships Between Combinatorial Dimensions} \label{app:prfrelation}

\subsection{Proof of (i) in Theorem \ref{thm:relation}.}
Fix $p \geq 2$ and $\gamma \in (0, \frac{1}{p}]$.  We first prove $\texttt{MS}_{\gamma}(\mathcal{H}) \leq \texttt{SL}(\mathcal{H})$. Let   $\Tcal$ be a $\Pi(\Ycal)$-ary tree of depth $d_{\gamma } = \texttt{MS}_{\gamma}(\mathcal{H}) $ shattered by $\Hcal$. For each internal node $v$ in $\Tcal$, keep the outgoing edges indexed by $\{\delta_y\}_{y \in \Ycal}$, where $\delta_y$ is a Dirac measure with point mass on $y$, and remove all other edges. Let $A_y$ be the set labeling the outgoing edge from $v$ indexed by $\delta_y$. Since $\delta_y(A_y) \leq 1-\gamma$, we have $y \notin A_y$. Changing the index of edges from $\delta_y$ to $y$ for all the remaining outgoing edges, we obtain a $\Ycal$-ary tree of depth $d_{\gamma}$. Repeating this process of pruning and reindexing recursively for every internal node, a $\Pi(\Ycal)$-ary tree shattered by $\Hcal$ can be transformed into a $\Ycal$-ary tree of the same depth shattered by $\Hcal$. Thus, we must have $\texttt{MS}_{\gamma}(\mathcal{H}) \leq \texttt{SL}(\mathcal{H}) $ for all $\gamma \in (0, \frac{1}{p}]$. For $\gamma = 0$,  the shattering condition gives $\delta_{y}(A_y) < 1$, which implies that $y \notin A_y$. The rest of the arguments are identical to case $\gamma \in (0,\frac{1}{p}]$ presented above. Therefore, $\texttt{MS}_{\gamma}(\mathcal{H}) \leq \texttt{SL}(\mathcal{H}) $ for all $\gamma \in [0, \frac{1}{p}]$.

We now prove $ \texttt{SL}_p(\Hcal) \leq \texttt{MS}_{\gamma}(\Hcal)$ for $\gamma \in [0, \frac{1}{p}]$. Let $\Tcal$ be a $[p]$-ary tree shattered by $\Hcal$.  We expand $\Tcal$ to obtain a $\Pi(\Ycal)$-ary tree of depth $d$ at scale $\frac{1}{p}$. Let $v$ be the root node in $\Tcal$, and $A_1, ..., A_p$ be the labels on the outgoing edges from $v$. To transform $\Tcal$ to a $\Pi(\Ycal)$-ary tree, we construct an outgoing edge for each measure. Fix a measure $\mu \in \Pi(\Ycal)$. There must be an $A_{\mu} \in \{A_1,...,A_p\}$ such that $\mu(A_{\mu}) \leq 1-\frac{1}{p}$. Suppose, for the sake of contradiction, this is not true. That is, $\mu(A_{i}) > 1-\frac{1}{p}$ for all $A_1, ..., A_p$, which further implies that $\mu(A_{i}^{c}) < \frac{1}{p}$. Since $\bigcap_{i=1}^p A_i = \emptyset$, we have $\Ycal=  \bigcup_{i=1}^p A^c_i$ and thus 
\[\mu(\Ycal) = \mu \left(\bigcup_{i=1}^p A^c_i\right) \leq \sum_{i=1}^p \mu(A_i^c) < 1,   \]
which contradicts the fact that $\mu$ is a probability measure. Therefore, for every $\mu$, there exists a $A_{\mu} \in \{A_1, ..., A_p\}$ such that $\mu(A_{\mu}) \leq 1- \frac{1}{p}$.   For every measure $\mu \in \Pi(\mathcal{Y})$, add an outgoing edge from $v$ indexed by $\mu$ and labeled by $A_{\mu}$. Pick the sub-tree in $\mathcal{T}$ following the outgoing edge from $v$ labeled by $A_{\mu}$ and append it to the newly constructed outgoing edge from $v$ indexed by $\mu$. Remove the two original outgoing edges from $v$ indexed by elements of $[p]$ and their corresponding subtree.
Upon repeating this process recursively for every internal node $v$ in $\Tcal$, we obtain a $\Pi(\Ycal)$-ary tree that is $\frac{1}{p}$-shattered by $\Hcal$. Thus, we have  $\texttt{MS}_{\frac{1}{p}}(\mathcal{H}) \geq \texttt{SL}_p(\mathcal{H})$. Using monotonicity of $\text{MSdim}$, we therefore conclude that $\texttt{MS}_{\gamma}(\mathcal{H}) \geq \texttt{SL}_p(\mathcal{H})$ for all $\gamma \in [0, \frac{1}{p}]$.

\subsection{Proof of (ii) in Theorem \ref{thm:relation}.} Let $p = \texttt{H}(\Scal(\Ycal)) < \infty$. Given $p\geq 2$ and (i), it suffices to show that $ \texttt{SL}_p(\mathcal{H}) \geq \texttt{MS}_{\gamma}(\mathcal{H}) \geq \texttt{SL}(\mathcal{H})$ for all $\gamma \in [0,\frac{1}{p}]$. We first show that $\texttt{MS}_{\gamma}(\mathcal{H}) \geq \texttt{SL}(\mathcal{H})$ for all $\gamma \in [0,\frac{1}{p}]$

  Consider a $\Ycal$-ary tree $\Tcal$ of depth $d = \texttt{SL}(\mathcal{H})$ shattered by $\Hcal$. Let $v$ be the root node of $\Tcal$, and $\{A_y\}_{y \in \Ycal}$ be the sequence of sets labeling the outgoing edges from $v$.  Since $p < \infty$, there must be a subsequence $\{A_{y_i}\}_{i=1}^p \subset \{A_y\}_{y \in \Ycal}$ such that $\cap_{i=1}^p A_{y_i} = \emptyset$. We keep the edges labeled by sets $\{A_{y_i}\}_{i=1}^p$ and remove all other edges, and repeat this process for every internal node $v$ in $\Tcal$. The subsequence of length $p$  may not be unique, but choosing arbitrarily is permissible. Upon repeating this process recursively for every internal node in the tree $\mathcal{T}$, we obtain a tree $\Tcal^{\prime}$ of width $p$ such that the sets labeling the $p$ outgoing edges from any internal node are mutually disjoint.  

Next, we expand $\Tcal^{\prime}$ to obtain a $\Pi(\Ycal)$-ary tree of depth $d$ at scale $\frac{1}{p}$. Let $v$ be the root node in $\Tcal^{\prime}$, and $\{A_{y_{i}}\}_{i=1}^p$ be the labels on the outgoing edges from $v$. To transform $\Tcal^{\prime}$ to a $\Pi(\Ycal)$-ary tree, we now construct an outgoing edge for each measure. Fix a measure $\mu \in \Pi(\Ycal)$. There must be an $i \in [p]$ such that $\mu(A_{y_i}) \leq 1-\frac{1}{p}$. Suppose, for the sake of contradiction, this is not true. That is, $\mu(A_{y_i}) > 1-\frac{1}{p}$ for all $ i \in [p]$, which further implies that $\mu(A_{y_i}^{c}) < \frac{1}{p}$. Since $\cap_{i=1}^p A_{y_i} = \emptyset$, we have $\Ycal=   \cup_{i=1}^p A_{y_i}^c$ and thus 
\[\mu(\Ycal) = \mu \left(\cup_{i=1}^p A_{y_i}^c\right) \leq \sum_{i=1}^p \mu(A_{y_i}^c) < \sum_{i=1}^p \frac{1}{p } < 1,   \]
which contradicts the fact that $\mu$ is a probability measure. Therefore, for every $\mu$, there exists a $y_{\mu} \in \{y_i\}_{i=1}^p$ such that $\mu(A_{y_{\mu}}) \leq 1- \frac{1}{p}$.   For every measure $\mu \in \Pi(\mathcal{Y})$, add an outgoing edge from $v$ indexed by $\mu$ and labeled by $A_{y_{\mu}}$. Pick the sub-tree in $\mathcal{T}^{\prime}$ following the outgoing edge from $v$ indexed by $y_{\mu}$ and append it to the newly constructed outgoing edge from $v$ indexed by $\mu$. Remove $p$ remaining outgoing edges from $v$ indexed by $y \in \{y_i\}_{i=1}^p$.
Upon repeating this process for every internal node $v$ in $\Tcal^{\prime}$, we obtain a $\Pi(\Ycal)$-ary tree that is $\frac{1}{p}$-shattered by $\Hcal$. Thus, we have  $\texttt{MS}_{\frac{1}{p}}(\mathcal{H}) \geq \texttt{SL}(\mathcal{H})$. Using monotonicity of $\text{MSdim}$, we therefore conclude that $\texttt{MS}_{\gamma}(\mathcal{H}) \geq \texttt{SL}(\mathcal{H})$ for all $\gamma \in [0, \frac{1}{p}]$. 

We now prove that $\texttt{SL}_p(\mathcal{H}) \geq \texttt{MS}_{\gamma}(\mathcal{H})$. Suppose $\mathcal{T}$ is a $\Pi(\mathcal{Y})$-ary tree $\gamma$-shattered by $\mathcal{H}$ according to Definition \ref{def:msdim}. Let $v$ be the root node of $\mathcal{T}$. Let $A_y$ be the set labeling the outgoing edge from $v$ indexed by $\delta_y$. Since $\delta_y(A_y) \leq 1 - \gamma$, we have that $y \notin A_y$. Therefore, $\bigcap_{y \in \mathcal{Y}} A_y = \emptyset$. Since $p < \infty$, there must be a subsequence $\{A_{y_i}\}_{i=1}^p \subset \{A_y\}_{y \in \Ycal}$ such that $\bigcap_{i=1}^p A_{y_i} = \emptyset$. Keep the outgoing edges indexed by $\{\delta_{y_i}\}_{i=1}^p$ and remove all other edges along with their subtrees. For each $i \in [p]$, change the index $\delta_{y_i}$ to $i$. The root node $v$ should now have $p$ outgoing edges, where each edge is indexed by a unique element $i \in [p]$ and labeled by the set $A_{y_i}$ such that $\bigcap_{i=1}^p A_{y_i} = \emptyset$. Repeat this process recursively on the subtrees following the $p$ reindexed edges results into a $\texttt{SL}_{p}$ tree of depth $d_{\gamma}$ shattered by $\mathcal{H}$. Thus, $\texttt{SL}_{p}(\mathcal{H}) \leq \texttt{MS}_{\gamma}(\mathcal{H})$ for $\gamma \in (0, \frac{1}{p}]$. The case when $\gamma = 0$ follows similarly.

\section{Deterministic Learnability in the Realizable Setting}\label{app:det_real}

\subsection{Upperbounds}
\begin{proof}(of upperbound in Theorem \ref{thm:det_real})
We first show that Algorithm \ref{alg:det_SOA} is a mistake-bound algorithm that makes at most $\texttt{SL}(\mathcal{H})$ mistakes on any realizable stream. To show this, we argue that (1) every time Algorithm $\ref{alg:det_SOA}$ makes a mistake, the SLdim of the version space goes down by $1$ and (2) if the SLdim of the current version space is $0$, then there is a prediction strategy such that the algorithm does not make any further mistakes. 

Let $t \in [T]$ be a round where Algorithm \ref{alg:det_SOA} makes a mistake, that is $\hat{y}_t \notin S_t$, and $\texttt{SL}(\Hcal) > 0$.  We  show that the $\texttt{SL}$ goes down by at least $1$, that is $\texttt{SL}(V_{t}) \leq \texttt{SL}(V_{t-1}) -1 $. For the sake of contradiction, assume that $\texttt{SL}(V_{t}) > \texttt{SL}(V_{t-1}) -1 $. As $\texttt{SL}(V_{t}) \leq \texttt{SL}(V_{t-1}) $, we must have $\texttt{SL}(V_{t}) = \texttt{SL}(V_{t-1}) =: m$. Since the $\texttt{SL}$ did not go down and the algorithm made a mistake, the min-max prediction strategy implies that for every $y \in \Ycal$, there exists $A_y \in \Scal(\Ycal)$ such that $y \notin A_y$ and $\texttt{SL}(V_{t-1}(A_y)) = m$. Next, construct a 
 $\Ycal$-ary tree $\mathcal{T}$ with $x_t$ labeling the root node. For every $y \in \Ycal$, label the outgoing edge  indexed by $y$ with the set $A_y$. Append the $\Ycal$-ary tree of depth $m$ associated with version space $V_{t-1}(A_y)$ to the edge indexed by $y$. Note that the depth of tree $\Tcal$ must be $m+1$, thus implying $\texttt{SL}(V_{t-1}) = m+1 $, which is a contradiction. Therefore, it must be the case that $\texttt{SL}(V_{t}) \leq  \texttt{SL}(V_{t-1}) -1 $. 

Let $t^{\star} \in [T]$ be round when the algorithm makes its $\texttt{SL}(\mathcal{H})^{th}$ mistake.  If $t^{\star}$ does not exist, the algorithm makes at most $\texttt{SL}(\mathcal{H})-1$ mistakes. So, without loss of generality, consider the case when $t^{\star}$ exists. It now suffices to show that the algorithm makes no further mistakes. We have already shown that $\texttt{SL}(V_{t^{\star}}) =0 $. Next, we show that for any $t > t^{\star}$, there must exist $ y \in \Ycal$ such that for all $A \in \Scal_t(\Ycal)$ we have $y \in A $. Suppose, for the sake of contradiction, this is not true. That means, for all $y \in \Ycal$, there exists $A_{y} \in \Scal_t(\Ycal)$ such that $y \notin A_y$. Consider a tree with $x_t$ in the root node, and every edge indexed by $y \in \Ycal$ is labeled  with the set $A_y$. As $A_y \cap \{h(x_t) \, \mid \, h \in V_{t-1}\} \neq \emptyset $, for every $y$, there exists a hypothesis $h_y$ such that $h_y(x_t) \in A_y$. By definition of $\texttt{SL}$, this implies that $\texttt{SL}(V_{t-1}) \geq 1$, which contradicts the fact that $\texttt{SL}(V_{t^{\star}}) = 0 $. Thus, there must be a prediction strategy $y \in \Ycal$ such that for any set $S_t \in \Scal_t(\Ycal)$ that the adversary can reveal, $y \in S_t$. With the prediction strategy in step 4, the algorithm makes no further mistakes. \end{proof}

\subsection{Lowerbounds}

\begin{proof}(of lowerbound in Theorem \ref{thm:det_real})
    We now show that for any deterministic learner, there exists a realizable stream where the learner makes at least $\texttt{SL}(\mathcal{H}) = d$ mistakes. The stream is obtained by traversing  the Set Littlestone tree of depth $d$, adapting to the algorithm's prediction.  Let $\mathcal{T}$ be a complete $\mathcal{X}$-valued, $\mathcal{Y}$-ary tree of depth $d$ that is shattered by $\Hcal$. Let $(f_1, ..., f_d)$ be the sequence of edge-labeling functions  $f_t: \mathcal{Y}^{t} \rightarrow \mathcal{S}(\mathcal{Y})$ associated with $\Tcal$. Consider the stream $\{(\Tcal_1(\hat{y}_{< t}), f_{t}(\hat{y}_{\leq t}))\}_{t=1}^d$,
where $\Tcal_1(\hat{y}_{< 1}) $ is the root node of the tree, and $\hat{y} = (\hat{y}_1, \ldots, \hat{y}_d)$ is algorithm's prediction on rounds $1, 2, \ldots, d$. Note that we can use the learner's prediction on round $t$ to generate the true feedback for round $t$ because the learner is deterministic and its prediction on any instance can be simulated apriori. Since we have $\hat{y}_t \notin f_t(\hat{y}_{\leq t})$ for all $t \in [d]$ by the definition of the tree, the algorithm makes at least $d$ mistake in the stream above. Finally, the stream considered above is realizable because there exists $h_{\hat{y}}$ such that $h_{\hat{y}}(\Tcal_t(\hat{y}_{< t})) \in f_t(\hat{y}_{\leq t})$ for all $t \in [d]$. This completes our proof. \end{proof}

\section{Randomized Learnability in the Realizable Setting} \label{app:rand_learn}

\subsection{Upperbounds}
\subsubsection{Fixed-scale Randomized Learner}

We give a fixed-scale learner in the realizable setting and prove a guarantee on its expected number of mistakes. In particular, we show that the expected mistake bound of Algorithm \ref{alg:rand_SOA}, for any fixed input scale $\gamma > 0$, is at most $\gamma T + \texttt{MS}_{\gamma}(\mathcal{H})$ on any realizable stream.






    
    


\begin{algorithm}

\setcounter{AlgoLine}{0}
\caption{Randomized Standard Optimal Algorithm (RSOA)}\label{alg:rand_SOA}
\KwIn{$\mathcal{H}$, Target accuracy $\varepsilon > 0$}
Initialize $V_{0} = \mathcal{H}$

\For{$t = 1,...,T$} {
     Receive unlabeled example $x_t \in \mathcal{X}$.

    For each $A \in \Scal(\Ycal)$, define $V_{t-1}(A) := \{h \in V_{t-1} \, \mid \, h(x_t) \in A\}$. 

    Let $\mathcal{S}_t(\mathcal{Y}) := \{A \in \mathcal{S}(\mathcal{Y}): A\cap \{h(x_t) \, \mid \, h \in V_{t-1}\} \neq \emptyset\}.$

    If $\texttt{MS}_{\varepsilon}(V_{t-1}) = 0$, let $\hat{\mu}_t \in \Pi(\Ycal)$ be such that for all $A \in \Scal_t(\Ycal)$ we have $\hat{\mu}_t(A) > 1-\varepsilon.$ 
    
    Else, compute
    \[\hat{\mu}_t = \argmin_{\mu \in \Pi(\Ycal)} \,\max_{\substack{A \in \Scal(\Ycal) \\ \mu(A) \leq 1 - \varepsilon}} \, \texttt{MS}_{\varepsilon}(V_{t-1}(A)).\]
    
    Predict $\hat{y}_t \sim \hat{\mu}_t.$
    
    Receive feedback $S_t$ and update $V_t = V_{t-1}(S_t).$
  
}
\end{algorithm}

\begin{lemma}[Fixed-scale Randomized Learning Guarantee] \label{lem:fixedscale}
\noindent For any $\mathcal{S}(\mathcal{Y}) \subseteq \sigma(\mathcal{Y})$,  $\mathcal{H} \subseteq \mathcal{Y}^{\mathcal{X}}$, and any input scale $\gamma > 0$, the expected cumulative loss of Algorithm \ref{alg:rand_SOA}, on any realizable stream, is  $\leq \gamma T + \emph{\texttt{MS}}_{\gamma}(\mathcal{H})$. 
\end{lemma}

\begin{proof}
We show that given any target accuracy $\varepsilon > 0$, the expected cumulative loss of Algorithm \ref{alg:rand_SOA} is at most $d_{\varepsilon} + \varepsilon T$ on any realizable stream, where $d_{\varepsilon} = \texttt{MS}_{\varepsilon}(\mathcal{H})$. In fact, we show that Algorithm \ref{alg:rand_SOA} achieves an even \textit{stronger} guarantee, namely that on any realizable sequence $\{(x_t, S_t)\}_{t=1}^T$, Algorithm \ref{alg:rand_SOA} computes distributions $\hat{\mu}_t \in \Pi(\mathcal{Y})$ such that

\begin{equation}\label{eq:randSOA}
\sum_{t=1}^T \mathbbm{1}\{\hat{\mu}_t(S^c_t) \geq \varepsilon\} \leq d_{\varepsilon}.
\end{equation}

From here, it follows that $\mathbbm{E}\left[\sum_{t=1}^T \mathbbm{1}\{\hat{y}_t \notin S_t\}\right] \leq d_{\varepsilon} + \varepsilon T.$ To see this, observe that

\begin{align*}
\mathbbm{E}\left[\sum_{t=1}^T \mathbbm{1}\{\hat{y}_t \notin S_t\}\right] &= \sum_{t=1}^T \mathbbm{P}\left[\hat{y}_t \notin S_t\ \right]\\
&= \sum_{t=1}^T \mathbbm{P}\left[\hat{y}_t \notin S_t\ \right]\mathbbm{1}\{\hat{\mu}_t(S^c_t) \geq \varepsilon\} + \mathbbm{P}\left[\hat{y}_t \notin S_t\ \right]\mathbbm{1}\{\hat{\mu}_t(S^c_t) < \varepsilon\}\\
&\leq \sum_{t=1}^T \mathbbm{1}\{\hat{\mu}_t(S^c_t) \geq \varepsilon\} + \varepsilon T \\
&\leq d_{\varepsilon} + \varepsilon T
\end{align*}

We now show that the outputs of Algorithm \ref{alg:rand_SOA} satisfy Equation \eqref{eq:randSOA}. It suffices to show that (1) on any round where $\hat{\mu}_t(S_t) \leq 1 - \varepsilon$ and $\texttt{MS}_{\varepsilon}(V_{t-1}) > 0$, we have $\texttt{MS}_{\varepsilon}(V_t) \leq \texttt{MS}_{\varepsilon}(V_{t-1}) - 1$, and (2) if $\texttt{MS}_{\varepsilon}(V_{t-1}) = 0$ then there is always a  distribution $\hat{\mu}_t \in \Pi(\mathcal{Y})$ such that $\mathbbm{P}\left[\hat{y}_t \notin S_t \right] \leq \varepsilon.$

Let $t \in [T]$ be a round where $\hat{\mu}_t(S_t) \leq 1 - \varepsilon$ and $\texttt{MS}_{\varepsilon}(V_{t-1}) > 0$.  For the sake contradiction, suppose that $\texttt{MS}_{\varepsilon}(V_t) = \texttt{MS}_{\varepsilon}(V_{t-1}) = d$. Then, by the min-max computation in line (4) of Algorithm \ref{alg:rand_SOA}, for every measure $\mu \in \Pi(\mathcal{Y})$, there exists a subset $\mathcal{A}_{\mu} \in \mathcal{S}(\mathcal{Y})$ such that  $\mu(A_{\mu}) \leq 1 - \varepsilon$ and $\texttt{MS}_{\varepsilon}(V_{t-1}(A_{\mu})) = d$. Now construct a tree $\mathcal{T}$ with $x_t$ labeling the root node. For each measure $\mu \in \Pi(\mathcal{Y})$, construct an outgoing edge from $x_t$ indexed by $\mu$ and  labeled by $A_{\mu}$. Append the  tree of depth $d$ associated with the version space $V_{t-1}(A_{\mu})$ to the edge indexed by $\mu$. Note that the depth of $\mathcal{T}$ must be $d+1$. Therefore, by definition of MSdim, we have that $\texttt{MS}_{\varepsilon}(V_{t-1}) = d+1$, a contradiction.  Thus, it must be the case that $\texttt{MS}_{\varepsilon}(V_t) \leq \texttt{MS}_{\varepsilon}(V_{t-1}) - 1$. 

Now, suppose $t \in [T]$ is a round such that $\texttt{MS}_{\varepsilon}(V_{t-1}) = 0$. We show that there always exist a distribution $\hat{\mu}_t \in \Pi(\mathcal{Y})$ such that for all $A \in \Scal_t(\Ycal)$ , we have $\hat{\mu}_t(A) \geq 1-\varepsilon$. Since we are in the realizable setting, it must be the case that $S_t \in \mathcal{S}_t(\mathcal{Y})$. Therefore, $\hat{\mu}_t(S_t) \geq 1 - \varepsilon$ and $\mathbbm{P}\left[\hat{y}_t \notin S_t \right] \leq \varepsilon$ as needed. To see why such a $\hat{\mu}_t$ must exist, suppose for the sake of contradiction that it does not exist. Then, for all $\mu \in \Pi(\mathcal{Y})$, there exists a set $A_{\mu} \in \mathcal{S}_t(\mathcal{Y})$ such that $\mu(A_{\mu}) \leq 1 - \varepsilon$. As before, consider a tree with root node labeled by $x_t$. For each measure $\mu \in \Pi(\mathcal{Y})$, construct an outgoing edge from $x_t$ indexed by $\mu$ and labeled by $A_{\mu}$. Since $A_{\mu} \cap \{h(x_t) \, \mid \, h \in V_{t-1}\} \neq \emptyset$, there exists a hypothesis $h_{\mu}$ such that $h_{\mu}(x_t) \in A_{\mu}$. By definition of MSdim, this implies that $\texttt{MS}_{\varepsilon}(V_{t-1}) \geq 1$, which contradicts the fact that $\texttt{MS}_{\varepsilon}(V_{t-1}) = 0$. Thus, there must be a distribution $\hat{\mu}_t \in \Pi(\mathcal{Y})$ such that for any set $A \in \mathcal{S}_t(\mathcal{Y})$, we have $\hat{\mu}_t(A) \geq 1 - \varepsilon$. Since this is precisely the distribution that Algorithm \ref{alg:rand_SOA} plays in step (3) and since $\texttt{MS}_{\varepsilon}(V_{t'}) \leq \texttt{MS}_{\varepsilon}(V_{t-1})$ for all $t' \geq t$, the algorithm no longer suffers expected loss more than $\varepsilon$. This completes the proof of Lemma \ref{lem:fixedscale}. 
\end{proof}

 We point out that \cite{filmus2023optimal} also considers a randomized online learner in the realizable setting that shares similarities with Algorithm \ref{alg:rand_SOA}. In particular, their algorithm also maintains a version space and optimizes over probability distributions.  However, they only consider binary classification and use a different complexity measure.  Moreover, the idea of optimizing over probability distributions on a measurable space should also remind  the reader of the generic min-max algorithm proposed by \cite{rakhlin_relax_n_randomize}.

\subsubsection{Multi-scale Randomized Learner}

The RSOA (Algorithm \ref{alg:rand_SOA}) runs at a fixed, pre-determined scale $\gamma \in [0, 1]$.  In this section, we upgrade this result by adapting the technique from \cite{daskalakis2022fast} to give a randomized, \textit{multi-scale} online learner (Algorithm \ref{alg:multiscale}) in the realizable setting. Lemma \ref{lem:multiscale} presents the main result,  which bounds the expected cumulative loss of Algorithm \ref{alg:multiscale} on any realizable data stream and gives the upperbound stated in Theorem  \ref{thm:rand_real}. 

\begin{lemma}[Multi-scale Randomized Online Learner]\label{lem:multiscale}
\noindent For any $\mathcal{S}(\mathcal{Y}) \subseteq \sigma(\mathcal{Y})$ and $\mathcal{H} \subseteq \mathcal{Y}^{\mathcal{X}}$, the expected cumulative loss of Algorithm \ref{alg:multiscale} on any realizable stream is at most

$$C \inf_{\gamma \in [0,1]}\Bigg\{\gamma T + \int_{\gamma}^1 \emph{\texttt{MS}}_{\eta}(\mathcal{H}) \, d\eta\Bigg\},$$

\noindent for some universal constant $C > 0$. 
\end{lemma}

 We highlight that the guarantee given by Lemma \ref{lem:multiscale} is analogous to Dudley’s integral entropy bound in batch setting and also matches Theorem 1 in \cite{daskalakis2022fast}. Compared to Lemma \ref{lem:fixedscale}, the upperbound given by Lemma \ref{lem:multiscale} can be significantly better. For example,  when the Measure Shattering dimension exhibits growth $\texttt{MS}_{\gamma}(\mathcal{H}) \approx \gamma^{-p}$ for some $p \in (0, 1)$, the bound given by Lemma \ref{lem:multiscale} is constant $O(1)$, while the bound given by Lemma  \ref{lem:fixedscale} scales according to $T^{\frac{p}{(1 + p)}}.$

The main algorithmic idea needed to obtain the guarantee in Lemma \ref{lem:multiscale} is to figure out how to make predictions using more than one scale. At a high-level, our multi-scale learner uses a sequence of $N$ scales $\{\gamma_i\}_{i=1}^N$, where $\gamma_i = \frac{1}{2^i}$, to compute a sequence of measures $\{\mu_t^i\}_{i=1}^N \subset  \Pi(\mathcal{Y})$ in each round $t \in [T]$. Then, our multi-scale learner uses the Measure Selection Procedure, defined in Algorithm \ref{alg:measuresel}, to carefully select one of the measures $\hat{\mu}_t \in \{\mu_i^t\}_{i=1}^N$ and makes a prediction $\hat{y}_t \sim \hat{\mu}_t$.


\begin{algorithm}

\setcounter{AlgoLine}{0}
\caption{Measure Selection Procedure (MSP)}\label{alg:measuresel}
\KwIn{Sequence of measures $\mu_1, ..., \mu_N$, valid sets $\mathcal{S} \subseteq \sigma(\mathcal{Y})$}

    If there exists a $m \in \mathbbm{N}$ such that for all $2 \leq i \leq m$, we have:

    \[\sup_{A \in \mathcal{S}} |\mu_i(A^c) - \mu_{i-1}(A^c)| \leq 2\gamma_{i-1} \quad \quad \text{ but } \quad \quad  \inf_{A \in \mathcal{S}} |\mu_m(A^c) - \mu_{m+1}(A^c)| \geq 2\gamma_{m} \]

    return $m$.
    
    Else, return $N$.
\end{algorithm}

Once the true label set is revealed, the multi-scale learner updates its self in the exact same way as RSOA.  Algorithm \ref{alg:multiscale}  formalizes the idea above.

\begin{algorithm}

\setcounter{AlgoLine}{0}
\caption{Multi-scale Online Learner (MSOL) }\label{alg:multiscale}
\KwIn{\textbf{Input:} $\mathcal{H}$, number of scales $N$
}

\textbf{Initialize:} $V_0 = \mathcal{H}$, $\gamma_i = \frac{1}{2^i}$ for $i \in [N]$

\For{$t = 1,...,T$} {

    Receive unlabeled example $x_t \in \mathcal{X}$.

    For each $A \in \Scal(\Ycal)$, define $V_{t-1}(A) := \{h \in V_{t-1} \, \mid \, h(x_t) \in A\}$. 

    Let $\mathcal{S}_t(\mathcal{Y}) := \{A \in \mathcal{S}(\mathcal{Y}): A\cap \{h(x_t) \, \mid \, h \in V_{t-1}\} \neq \emptyset\}.$

   \uIf{$\emph{\texttt{MS}}_{\gamma_N}(V_{t-1}) = 0$}
   {
    Let $\hat{\mu}_t \in \Pi(\Ycal)$ such that  $\hat{\mu}_t(A) > 1-\gamma_N$  for all $ A \in \Scal_t(\Ycal)$.}

    \uElse{
    
    \For{ i = 1, \ldots, N}{
     
            If $\texttt{MS}_{\gamma_i}(V_{t-1}) = 0$, let $\mu^i_t \in \Pi(\Ycal)$  such that $\mu^i_t(A) > 1-\gamma_i$ for all $A \in \Scal_t(\Ycal)$.
            
            Else, let
                \[\mu^i_t = \argmin_{\mu \in \Pi(\Ycal)} \,\max_{\substack{A \in \Scal(\Ycal) \\ \mu(A) \leq 1 - \gamma_i}} \, \texttt{MS}_{\gamma_i}(V_{t-1}(A)).\]
 
            }
        Compute $m_t = \text{MSP}(\{\mu_t^i\}_{i=1}^N, \mathcal{S}_t(\mathcal{Y}))$ and let $\hat{\mu}_t = \mu_t^{m_t}$.
        }

    Predict $\hat{y}_t \sim \hat{\mu}_t.$
    
    Receive feedback $S_t\in \Scal_t(\Ycal)$ and update $V_t = V_{t-1}(S_t).$
}
\end{algorithm}

We now prove Lemma \ref{lem:multiscale}, which closely follows the analysis by \cite{daskalakis2022fast}.

\begin{proof}
    Fix a $N \in \mathbbm{N}$. Our first goal is to show that on any realizable stream, the expected cumulative loss of Algorithm \ref{alg:multiscale} is at most

    $$\gamma_N \, T + 16\sum_{i=1}^N \gamma_i \cdot \texttt{MS}_{\gamma_i}(\mathcal{H}),$$

    where $\gamma_i = \frac{1}{2^i}$. To that end, let $\{(x_t, S_t)\}_{t=1}^T$ denote the realizable stream that is to be observed by the learner. For all $t \in [T+1]$, define the potential function 

    $$\Phi_t = (T +1 - t)\gamma_N + 16\sum_{i=1}^N\gamma_i \texttt{MS}_{\gamma_i}(V_{t-1}).$$

    It suffices to show that $\Phi_t - \Phi_{t+1} \geq \hat{\mu}_t(S_t^c)$ for all $t \in [T]$. To see why this is sufficient, observe that summing over all $t \in [T]$ gives 

    \begin{align*}
    \sum_{t=1}^T \hat{\mu}_t(S_t^c) \leq \sum_{t=1}^T (\Phi_t - \Phi_{t+1})&= \Phi_1 - \Phi_{T+1}\leq T\gamma_N + 16\sum_{i=1}^N\gamma_i \texttt{MS}_{\gamma_i}(\mathcal{H})
    \end{align*}

where the inequality follows from the fact that $\Phi_{T+1} \geq 0$ and $V_0 = \mathcal{H}$. Finally, noting that $\mathbbm{E}_{\hat{y}_t \sim \hat{\mu}_t}\left[\mathbbm{1}\{\hat{y}_t \notin S_t\}\right] = \hat{\mu}_t(S_t^c)$ gives $\mathbbm{E}\left[ \sum_{t=1}^T \mathbbm{1}\{\hat{y}_t \notin S_t\}\right] \leq T\gamma_N + 16\sum_{i=1}^N\gamma_i \texttt{MS}_{\gamma_i}(\mathcal{H})$ as desired. 

The rest of this proof is dedicated to showing that $\Phi_t - \Phi_{t+1} \geq \hat{\mu}_t(S_t^c)$ for all $t \in [T]$. Fix a $t \in [T]$. Using the definition of $\Phi_t$, we need to show that

\begin{equation}\label{eq:multiscale}
\gamma_N + 16\sum_{i=1}^N \gamma_i (\texttt{MS}_{\gamma_i}(V_{t-1}) - \texttt{MS}_{\gamma_i}(V_{t})) \geq \hat{\mu}_t(S_t^c).
\end{equation}

If $\hat{\mu}_t(S_t^c) < \gamma_N$, then Inequality \ref{eq:multiscale} holds since for all $t \in [T]$ and $i \in [N]$, $\texttt{MS}_{\gamma_i}(V_{t-1}) \geq \texttt{MS}_{\gamma_i}(V_{t}).$ Thus, we focus on the case where $\hat{\mu}_t(S_t^c) \geq \gamma_N$.  

Suppose $\hat{\mu}_t(S_t^c) \geq \gamma_N$. Then,  $\texttt{MS}_{\gamma_N}(V_{t-1}) \geq 1$, the for-loop on line 5(a) runs, and the measure $\hat{\mu}_t = \mu_t^{m_t}$ computed on line 5(b) is used to make a prediction. This is because when  $\texttt{MS}_{\gamma_N}(V_{t-1}) = 0$, we are guaranteed the existence of a measure $\hat{\mu}_t \in \Pi(\mathcal{Y})$ such that $\hat{\mu}_t(S_t^c) < \gamma_N$ (see proof of Theorem \ref{thm:rand_real}) and by line 4, this would have precisely been the measure the learner uses to make its prediction.

 We now show that when $\hat{\mu}_t(S_t^c) \geq \gamma_N$, there exists an index $j \in [N]$ such that $\gamma_j \geq \frac{\hat{\mu}_t(S_t^c)}{16}$ and $\mu_t^j(S_t^c) \geq \gamma_j$. This implies Inequality (2), because if $\mu_t^j(S_t^c) \geq \gamma_j$, then $\texttt{MS}_{\gamma_j}(V_{t-1}) \geq 1$, and $\texttt{MS}_{\gamma_j}(V_{t}) < \texttt{MS}_{\gamma_j}(V_{t-1})$, which follows from the definition of MSdim, and the min-max prediction strategy in step 5(a:ii). Then, we can compute
 \begin{align*}
 \gamma_N + 16\sum_{i=1}^N \gamma_i (\texttt{MS}_{\gamma_i}(V_{t-1}) - \texttt{MS}_{\gamma_i}(V_{t})) &\geq 16 \gamma_j (\texttt{MS}_{\gamma_j}(V_{t-1}) - \texttt{MS}_{\gamma_j}(V_{t})) \geq \hat{\mu}_t(S_t^c),
 \end{align*}
 which matches the guarantee of Inequality \ref{eq:multiscale}. Accordingly, the rest of the proof will focus on showing the existence of such an index $j \in [N]$.  To do so, let $k \in \mathbbm{N}$ denote the smallest natural number such that $\hat{\mu}_t(S_t^c) \geq \gamma_k = \frac{1}{2^k}$. By definition of $k$, we have that $\hat{\mu}(S_t^c) < \gamma_{k-1} = 2\gamma_k$. Note that $k \neq N+1$ since that would imply that  $\hat{\mu}(S_t^c) < \frac{1}{2^N} = \gamma_N$ which contradicts the fact that $\hat{\mu}_t(S_t^c) \geq \gamma_N$. Thus, it must be the case that $k \in \{1, ..., N\}$. Let $m_t = \text{MSP}(\{\mu_t^i\}_{i=1}^N, \mathcal{S}_t(\mathcal{Y}))$ denote the index output by MSP in round $t$. We consider two subcases: (1) $k \in \{m_t + 1, ..., N\}$ and (2) $k \in \{1, ..., m_t\}$. 

 \textbf{Case I.} Suppose $k \in \{m_t + 1, ..., N\}$. Then, we show that $j = m_t+1$. That is, $\gamma_{m_t+1} \geq \frac{\hat{\mu}_t(S_t^c)}{16}$ and $\mu_t^{m_t+1}(S_t^c) \geq \gamma_{m_t+1}$ Recall that $\hat{\mu}_t(S_t^c) = \mu_t^{m_t}(S_t^c)$. Since $m_t < N$, by definition, we have that $\inf_{A \in \mathcal{S}_t(\mathcal{Y})}|\mu_t^{m_t}(A) - \mu_t^{m_t + 1}(A)| \geq 2\gamma_{m_t}$. This implies that  $|\mu_t^{m_t}(S_t^c) - \mu_t^{m_t + 1}(S_t^c)| \geq 2\gamma_{m_t}$. Moreover, we have that $\mu_t^{m_t}(S_t^c) = \hat{\mu}_t(S_t^c) < 2\gamma_k \leq 2\gamma_{m_t+1} = \gamma_{m_t}$. Combining the two inequalities, we get that $\mu_t^{m_t + 1}(S_t^c) \geq \gamma_{m_t} > \gamma_{m_t + 1}$. Since $\hat{\mu}_t(S_t^c) < 2\gamma_{m_t + 1}$, we also obtain $\gamma_{m_t + 1} \geq \frac{\hat{\mu}_t(S_t^c)}{2} > \frac{\hat{\mu}_t(S_t^c)}{16}$. This completes this case. 

Now, suppose that $k \in \{1, ..., m_t\}$. Then we know that $\mu_t^{m_t}(S_t^c) = \hat{\mu}_t(S_t^c) \geq \gamma_k \geq \gamma_{m_t}$. We further break this case down into two subcases: (a) $k \in \{m_t - 3, m_t - 2, ..., m_t\}$ and (b) $k \in \{1, ..., m_t-4\}$. 

\textbf{Case II(a).} Consider the case where $k \in \{m_t - 3, m_t - 2, ..., m_t\}$. We show that $j = m_t$. We know that $\hat{\mu}_t(S_t^c) < 2 \gamma_k = 2\frac{1}{2^k} = 16\gamma_{k + 3} \leq 16\gamma_{m_t}$. This implies that $\gamma_{m_t} \geq \frac{\hat{\mu}_t(S_t^c)}{16}$. Since we have that $\mu_t^{m_t}(S_t^c) = \hat{\mu}_t(S_t^c) \geq \gamma_{m_t}$, this completes the proof that $j = m_t$. 

\textbf{Case II(b).} Consider the case where $k \in \{1, ..., m_t-4\}$. Here, we will show that $j = k + 1$. Observe that, 
\begin{align*}
|\mu_t^{m_t}(S_t^c) - \mu_t^{k+3}(S_t^c)| \leq \sum_{i = k+3}^{m_t - 1}|\mu_t^{i}(S_t^c) - \mu_t^{i+1}(S_t^c)|
&\leq \sum_{i=k+3}^{m_t - 1} 2 \gamma_i\\
&\leq 2 \sum_{i=k+3}^{\infty} \frac{1}{2^i}
= 4\gamma_{k+3} = \frac{\gamma_k}{2}, 
\end{align*}
where the second inequality follows from the definition of $m_t = \text{MSP}(\{\mu_t^i\}_{i=1}^N, \mathcal{S}_t(\mathcal{Y}))$. This implies that $\mu_t^{m_t}(S_t^c) - \mu_t^{k+3}(S_t^c) \leq \frac{\gamma_k}{2}$. Since $\mu_t^{m_t}(S_t^c) \geq \gamma_k$, we get that $\mu_t^{k+3}(S_t^c) \geq \frac{\gamma_k}{2} = 4\gamma_{k+3} \geq \gamma_{k+3}$. Finally, recall that $\hat{\mu}_t(S_t^c) < 2\gamma_k = 16 \gamma_{k+3}$, implying that $\gamma_{k+3} \geq \frac{\hat{\mu}_t(S_t^c)}{16}$ as desired. This completes the subcase. 

Overall, we have shown that when $\hat{\mu}_t(S_t^c) \geq \gamma_N$, there exists an index $j \in [N]$ such that $\gamma_j \geq \frac{\hat{\mu}_t(S_t^c)}{16}$ and $\mu_t^j(S_t^c) \geq \gamma_j$. This means that for all $t \in [T]$, $\Phi_t - \Phi_{t+1} \geq \hat{\mu}_t(S_t^c)$ and therefore the expected cumulative loss of Algorithm \ref{alg:multiscale} is at most $\gamma_N \, T + \sum_{i=1}^N \gamma_i \cdot \texttt{MS}_{\gamma_i}(\mathcal{H}),$ as needed.

Our next goal is to show that if $\gamma^* = \inf_{\gamma > 0}\{\gamma T + \int_{\gamma}^{1} \texttt{MS}_{\eta}(\mathcal{H}) d\eta\}$, then setting $N = \ceil{\frac{1}{\log 2 \gamma^*}}$ gives that 

$$\gamma_N \, T + 16\sum_{i=1}^N \gamma_i \cdot \texttt{MS}_{\gamma_i}(\mathcal{H}) \leq C\inf_{\gamma > 0}\Bigg\{\gamma T + \int_{\gamma}^1 \texttt{MS}_{\eta}(\mathcal{H}) d\eta\Bigg\}$$

for some constant $C > 0$. However, this follows from the fact that when $N = \ceil{\frac{1}{\log 2 \gamma^*}}$, $\gamma_N \leq 2\gamma^*$ and the fact that $16\sum_{i=1}^N \gamma_i \cdot \texttt{MS}_{\gamma_i}(\mathcal{H})$ is, up to a constant factor, the appropriate lower Reimann sum such that $16 \sum_{i=1}^N \gamma_i \cdot \texttt{MS}_{\gamma_i}(\mathcal{H}) \leq C \int_{\gamma^*}^1 \texttt{MS}_{\eta}(\mathcal{H}) d\eta$.
\end{proof}

\subsection{Lowerbounds}

In this section, we prove the lowerbound given in Theorem \ref{thm:rand_real}.  Fix $\gamma > 0$. Let $\mathcal{H}$ and $\mathcal{S}(\mathcal{Y})$ be such that $\texttt{MS}_{\gamma}(\mathcal{H}) = d_{\gamma}$. By definition of MSdim, there exists a $\mathcal{X}$-valued, $\Pi(\mathcal{Y})$-ary tree $\mathcal{T}$ of depth $d_{\gamma}$ shattered by $\mathcal{H}$. Let $(f_1, ..., f_d)$ be the sequence of edge-labeling functions $f_t: \Pi(\mathcal{Y})^t \rightarrow \mathcal{S}(\mathcal{Y})$ associated with $\mathcal{T}$. Let $\mathcal{A}$ be any randomized learner for $\mathcal{H}$. Our goal will be to use $\mathcal{T}$ and its edge-labeling functions $(f_1, ..., f_d)$ to construct a hard realizable stream for $\mathcal{A}$ such that on every round, $\mathcal{A}$ makes a mistake with probability at least $\gamma$. This stream is obtained by traversing $\mathcal{T}$, adapting to the sequence of distributions output by $\mathcal{A}$. 

To that end, for every round $t \in [d_{\gamma}]$,  let $\hat{\mu}_t$ denote the distribution that $\mathcal{A}$ computes before making its prediction $\hat{y}_t$. Consider the stream $\{\left(\mathcal{T}_t(\hat{\mu}_{<t}), f_t(\hat{\mu}_{\leq t})\right)\}_{t=1}^{d_{\gamma}}$, where $\hat{\mu} = 
(\hat{\mu}_1, \ldots, \hat{\mu}_{d_{\gamma}})$ denotes the sequence of distributions output by $\Acal$. This stream is obtained by starting at the root of $\mathcal{T}$, passing $\mathcal{T}_1$ to $\mathcal{A}$, observing the distribution $\hat{\mu}_1$ computed by $\mathcal{A}$, passing the label $f_t(\hat{\mu}_{\leq 1})$ to $\mathcal{A}$, and then finally moving along the edge labeled by $\hat{\mu}_1$. This process then repeats $d_{\gamma} -1$ times until the bottom of $\mathcal{T}$ is reached. Note that we can observe and use the distribution computed by $\mathcal{A}$ on round $t$ to generate the true feedback because a randomized algorithm \textit{deterministically} maps a sequence of labeled instances to a distribution. Moreover the stream is realizable since by the definition of shattering, there exists a $h_{\hat{\mu}} \in \mathcal{H}$ such that $h_{\hat{\mu}}(\mathcal{T}_t(\hat{\mu}_{<t})) \in f_t(\hat{\mu}_{\leq t})$ for all $t \in [d_{\gamma}]$.

Now, we are ready to show that this stream is difficult for $\mathcal{A}$. By definition of the tree, for all $t \in [d_{\gamma}]$, we have that $\hat{\mu}_t(f_t(\hat{\mu}_{\leq t})) \leq 1 - \gamma$. Therefore, since $\mathcal{A}$ receives $f_t(\hat{\mu}_{\leq t})$ as feedback on round $t$, we have that $\mathbbm{P}\left[\mathcal{A}(\mathcal{T}_t(\hat{\mu}_{<t})) \notin f_t(\hat{\mu}_{\leq t})\right] = \mathbbm{P}_{\hat{y}_t \sim \hat{\mu}_t}\left[\hat{y}_t \notin f_t(\hat{\mu}_{\leq t})\right] = 1- \hat{\mu}_t(f_t(\hat{\mu}_{\leq t})) \geq \gamma$ for all $t \in [d_{\gamma}]$. Summing over all $t \in [d_{\gamma}]$ gives that 

$$\mathbbm{E}\left[\sum_{t=1}^{d_{\gamma}} \mathbbm{1}\{\mathcal{A}(\mathcal{T}_t(\hat{\mu}_{<t})) \notin f_t(\hat{\mu}_{\leq t})\} \right] = \sum_{t=1}^{d_{\gamma}} \mathbbm{P}\left[\mathcal{A}(\mathcal{T}_t(\hat{\mu}_{<t})) \notin f_t(\hat{\mu}_{\leq t})\right] \geq \gamma\, d_{\gamma}.$$

This shows that $\mathcal{A}$ makes at least $\gamma d_{\gamma}$ mistakes in expectation on the realizable stream $\{\left(\mathcal{T}_t(\hat{\mu}_{<t}), f_t(\hat{\mu}_{\leq t})\right)\}_{t=1}^{d_{\gamma}}$. Since our choice of $\gamma$ and the randomized algorithm $\mathcal{A}$ was arbitrary, this holds true for any $\gamma > 0$ and any randomized online learner. This completes the proof.   

\section{Agnostic Learnability}
\subsection{Agnostic Upperbound}
\begin{proof} (of (i) in Theorem \ref{thm:agn})
Let $(x_1, S_1), \ldots, (x_T, S_T) $ be the data stream.  Let $h^{\star} = \argmin_{h \in \mathcal{H}} \sum_{t=1}^T \mathbbm{1}\{h(x_t) \notin S_t\}$ be an optimal function in hind-sight.  For a target accuracy $\varepsilon > 0$,  let $d_{\varepsilon} = \texttt{MS}_{\varepsilon}(\Hcal)$. Given time horizon $T$, let  $L_T = \{L \subset [T]; |L| \leq d_{\varepsilon}\}$ denote the set of all possible subsets of $[T]$ with size at most $d_{\varepsilon}$. For every $L \in L_T$ define an expert $E_{L}$ such that 
\[E_L(x_t) := \text{RSOA}_{\varepsilon}(x_t \mid L_{<t}),\]
where $ L_{<t} = L \cap \{1,2, \ldots, t-1\}$ and $ \text{RSOA}_{\varepsilon}(x_t \mid L_{<t})$ is the prediction of the Randomized Standard Optimal Algorithm (RSOA), defined as Algorithm \ref{alg:rand_SOA}, running at scale $\varepsilon$ that has updated on labeled examples $\{(x_i, S_i)\}_{i \in L_{<t}}$. 
Let $\mathcal{E} = \bigcup_{L \in L_T} \{E_{L}\}$ denote the set of all Experts parameterized by subsets $L \in L_T$. Note that $|\mathcal{E}| = \sum_{i=0}^{d_{\varepsilon}} \binom{T}{i} \leq T^{d_{\varepsilon}}$. Finally, given our set of experts $\mathcal{E}$, we run the Randomized Exponential Weights Algorithm (REWA), denoted hereinafter as $\mathcal{P}$, over the stream $(x_1, S_1), ..., (x_T, S_T)$ with a learning rate $\eta = \sqrt{2 \ln(|\mathcal{E}|)/T}$. Let $B$ denote the random variable associated with the internal randomness of the RSOA. Then, conditioned on $B$,  Theorem 21.11 of \cite{ShwartzDavid} tells us that
\begin{align*}
   \sum_{t=1}^T \mathbb{E}\left[ \mathbbm{1}\{\mathcal{P}(x_t) \notin S_t \} \mid B\right] &\leq \inf_{E \in \mathcal{E}} \sum_{t=1}^T \mathbbm{1}\{E(x_t) \notin S_t\} + \sqrt{2T\ln(|\mathcal{E}|)}\\
    &\leq \inf_{E \in \mathcal{E}} \sum_{t=1}^T \mathbbm{1}\{E(x_t) \notin S_t\} + \sqrt{2d_{\varepsilon}T\ln(T)}, 
\end{align*}
where the second inequality follows because $|\mathcal{E}| \leq T^{d_{\varepsilon}}$. Taking expectations on both sides of the inequality above, we obtain
\[ \mathbb{E}\left[\sum_{t=1}^T  \mathbbm{1}\{\mathcal{P}(x_t) \notin S_t \} \right] \leq \expect \left[\inf_{E \in \mathcal{E}} \sum_{t=1}^T \mathbbm{1}\{E(x_t) \notin S_t\} \right]+ \sqrt{2d_{\varepsilon}T\ln(T)},\]
Here, we  have an expectation on the right-hand side because the Expert predictions are random.
Define $R^{\star} = \{t \in [T]\, \mid \, h^{\star}(x_t) \in S_t\}$ to be the part of the stream realizable by $h^{\star}$. Note that the set $R^{\star}$ is not random because the adversary is oblivious. Then, we have 
\begin{equation*}
    \begin{split}
       \inf_{E \in \mathcal{E}} \sum_{t=1}^T \mathbbm{1}\{E(x_t) \notin S_t\} &=  \inf_{E \in \mathcal{E}} \left( \sum_{t \in R^{\star}}\mathbbm{1}\{E(x_t) \notin S_t\} + \sum_{t \notin R^{\star}}\mathbbm{1}\{E(x_t) \notin S_t\} \right) \\
       &\leq \inf_{E \in \mathcal{E}}  \sum_{t \in R^{\star}}\mathbbm{1}\{E(x_t) \notin S_t\} +  \sum_{t \notin R^{\star}}\mathbbm{1}\{h^{\star}(x_t) \notin S_t\} \\
       &= \inf_{E \in \mathcal{E}}  \sum_{t \in R^{\star}}\mathbbm{1}\{E(x_t) \notin S_t\} + \inf_{h \in \Hcal}  \sum_{t =1}^T\mathbbm{1}\{h(x_t) \notin S_t\},
    \end{split}
\end{equation*}
where the first inequality above follows because $\mathbbm{1}\{h^{\star}(x_t) \notin S_t\} =1$ for all $t \in R^{\star}$. Thus, the expected cumulative loss of $\mathcal{P}$ is 
\begin{equation}\label{eqn:agn_intermediate}
    \mathbb{E}\left[\sum_{t=1}^T  \mathbbm{1}\{\mathcal{P}(x_t) \notin S_t \} \right] \leq \inf_{h \in \Hcal}  \sum_{t =1}^T\mathbbm{1}\{h(x_t) \notin S_t\}+ \expect \left[\inf_{E \in \mathcal{E}} \sum_{t \in R^{\star}} \mathbbm{1}\{E(x_t) \notin S_t\} \right]+ \sqrt{2d_{\varepsilon}T\ln(T)}
\end{equation}
Thus, it suffices to show that the second term on the right side of the inequality above is $\leq d_{\varepsilon} + \varepsilon T$. 

 To do so, we need some more notation. Let us define $\hat{\mu}_t = \mu$$\text{-RSOA}_{\varepsilon}(x_t \mid L)$ to be the measure returned by $\text{RSOA}_{\varepsilon}$, as described in step 4 and 5 of Algorithm \ref{alg:rand_SOA}, for $x_t$ given that the algorithm has been updated on examples of the time points $t \in L$. We say that $\mu\text{-RSOA}_{\varepsilon}$ makes a mistake on round $t$ if $\indicator\{\hat{\mu}_t(S^c_t) \geq \varepsilon\}=1 $. With this notion of the mistake, Equation \eqref{eq:randSOA} tells us that $\text{RSOA}_{\varepsilon}$, run and updated on any realizable sequence, makes at most $d_{\varepsilon}$ mistakes.   Since $\mu\text{-RSOA}_{\varepsilon}(x \mid L)$ is a deterministic mapping from the past examples to a probability measure in $\Pi(\Ycal)$, we can procedurally define and select a sequence of time points in $R^{\star}$ where $\mu\text{-RSOA}_{\varepsilon}$, had it run exactly on this sequence of time points,  would make mistakes at each time point. To that end, let
$$\tilde{t}_1 = \min \Big\{\, t \in R^{\star}: \hat{\mu}_t(S^c_t) \geq \varepsilon  \text{ where }  \hat{\mu}_t =\mu\text{-RSOA}_{\varepsilon}\big(x_t|\,\{\}\big)  \Big\}$$
be the earliest time point in $R^{\star}$, where a fresh, unupdated copy of $\mu\text{-RSOA}_{\varepsilon}$ makes a mistake, if it exists. Given $\tilde{t}_1$, we recursively define $\tilde{t}_i$ for $i > 1$ as 

$$\tilde{t}_i = \min \Big\{\, t \in R^{\star}: \hat{\mu}_t(S^c_t) \geq \varepsilon  \text{ where }  \hat{\mu}_t =\mu\text{-RSOA}_{\varepsilon}\big(x_t|\, \{\tilde{t}_1, \ldots, \tilde{t}_{i-1}\}\big)  \text{ and } t > \tilde{t}_{i-1} \Big\}$$
if it exists. That is, $\tilde{t}_i$ is the earliest timepoint after $\tilde{t}_{i-1}$ in $R^{\star}$ where $\mu\text{-RSOA}_{\varepsilon}$ having updated only on the sequence $(x_{\tilde{t}_1}, S_{\tilde{t}_1}), ..., (x_{\tilde{t}_{i-1}}, S_{\tilde{t}_{i-1}}))$ makes a mistake. We stop this process when we reach an iteration where no such time point in $R^{\star}$ can be found where $\mu\text{-RSOA}_{\varepsilon} $ makes a mistake. 

Using the definitions above, let $\tilde{t}_1, \tilde{t}_2 ..., $ denote the sequence of timepoints in $R^{\star}$ selected via this recursive procedure.  Define $L^{\star} = \{\tilde{t}_1, \tilde{t}_2..., \}$ and let $E_{L^{\star}}$ be the expert parametrized by the set of indices $L^{\star}$. The expert $E_{L^{\star}}$ exists because $R^{\star}$ is a part of the stream that is realizable to $h^{\star}$ and Equation \eqref{eq:randSOA} implies that $|L^{\star}| \leq d_{\varepsilon}$.  By definition of the expert, we have $E_{L^{\star}}(x_t) = \text{RSOA}_{\varepsilon}(x_t \mid L^{\star}_{<t})$ for all $t \in [T]$. Let us define $\hat{\mu}_t^{\star} = \mu\text{-RSOA}_{\varepsilon}(x_t \mid L_{<t}^{\star})$. Then, we have 
\begin{equation*}
    \begin{split}
        \inf_{E \in \mathcal{E}} &\sum_{t \in R^{\star}} \mathbbm{1}\{E(x_t) \notin S_t\}  \\
        &\leq \sum_{t \in R^{\star}} \mathbbm{1}\{E_{L^{\star}}(x_t) \notin S_t\} \\ 
        &=\sum_{t \in R^{\star}} \mathbbm{1}\{\text{RSOA}_{\varepsilon}(x_t \mid L_{<t}^{\star}) \notin S_t\}\,   \indicator\{\hat{\mu}_t^{\star}(S^c_t) < \varepsilon\} + \sum_{t \in R^{\star}} \mathbbm{1}\{\text{RSOA}_{\varepsilon}(x_t \mid L_{<t}^{\star}) \notin S_t\}\,   \indicator\{\hat{\mu}_t^{\star}(S^c_t) \geq \varepsilon\} \\
        &\leq \sum_{t \in R^{\star}}  \mathbbm{1}\{\text{RSOA}_{\varepsilon}(x_t \mid L_{<t}^{\star}) \notin S_t\}\,   \indicator\{\hat{\mu}_t^{\star}(S^c_t) < \varepsilon\}+  \sum_{t \in R^{\star}} \indicator\{\hat{\mu}_t^{\star}(S^c_t) \geq \varepsilon\} \\
        &\leq \sum_{t \in R^{\star}} \mathbbm{1}\{\text{RSOA}_{\varepsilon}(x_t \mid L_{<t}^{\star}) \notin S_t\}\,   \indicator\{\hat{\mu}_t^{\star}(S^c_t) < \varepsilon\} + d_{\varepsilon},
    \end{split}
\end{equation*}
where the last inequality follows from the definition of $L^{\star}$ and the fact that $|L^{\star}| \leq d_{\varepsilon}$. Since 
\[\expect\left[\mathbbm{1}\{\text{RSOA}_{\varepsilon}(x_t \mid L_{<t}^{\star}) \notin S_t\}\,   \indicator\{\hat{\mu}_t^{\star}(S^c_t) < \varepsilon\} \right] = \hat{\mu}_t^{\star}(S^c_t)\, \indicator\{\hat{\mu}_t^{\star}(S^c_t) < \varepsilon\}  \leq \varepsilon ,\]
we obtain
\[\expect \left[\inf_{E \in \mathcal{E}} \sum_{t \in R^{\star}} \mathbbm{1}\{E(x_t) \notin S_t\} \right] \leq \varepsilon |R^{\star}| + d_{\varepsilon} \leq \varepsilon T + d_{\varepsilon}.\]

Finally, plugging this bound in Equation \eqref{eqn:agn_intermediate} yields
\[  \mathbb{E}\left[\sum_{t=1}^T  \mathbbm{1}\{\mathcal{P}(x_t) \notin S_t \} \right] \leq \inf_{h \in \Hcal}  \sum_{t =1}^T\mathbbm{1}\{h(x_t) \notin S_t\}+d_{\varepsilon}+ \varepsilon T+  \sqrt{2d_{\varepsilon}T\ln(T)}.\]
Since $\varepsilon > 0$ is arbitrary, this completes our proof.
\end{proof}
\subsection{Agnostic Lowerbound}

\begin{proof}(of (ii) in Theorem \ref{thm:agn})
Let $d = \texttt{SL}_2(\Hcal)$ and $d_{\gamma} = \texttt{MS}_{\gamma}(\Hcal)$ for $\gamma \in [0, 1]$.
The lowerbound of $\sup_{\gamma  > 0 } \, \gamma\,  d_{\gamma}$ on the expected regret in the agnostic setting follows trivially from the lowerbound on the expected cumulative loss in the realizable setting (see (ii) in Theorem \ref{thm:rand_real}). Moreover, when $\sup_{\gamma > 0} d_{\gamma} = 0$, there is no non-negative lowerbound on the expected regret. Indeed, consider the case where $\mathcal{Y} = [5]$, $\mathcal{S}(\mathcal{Y}) = \{\{3, 4\}, \{4, 5\}\})$ , and $\mathcal{H} = \{h_1, h_2\}$, where $h_i$ is a constant hypothesis that always outputs $i$. Then, $\sup_{\gamma  > 0 } d_{\gamma} = 0$ trivially. However, the expected regret of the algorithm that always outputs $4$ is $-T$.

Next, we will focus on showing how the lowerbound of $\sqrt{\frac{d\,T}{8}}$ can be obtained. When $d=0$, the claimed lowerbound is $\max\big\{\sqrt{dT/8}\, , \,  \sup_{\gamma  > 0 } \,  d_{\gamma}\big\} = \sup_{\gamma  > 0 } \, \gamma\,  d_{\gamma}$, which we have already established. Let $d > 0$ and $\Tcal$ be a  $\texttt{SL}_2$ tree of depth $d$ shattered by $\Hcal$. With a binary tree $\Tcal$, we now use the technique from \cite{ben2009agnostic} to obtain the aforementioned lowerbound. 

Consider  $T = k\,d$ for some odd $k \in \naturals$. For $\sigma \in \{\pm 1\}^{T}$, define $\tilde{\sigma}_{i} = \text{sign}\left( \sum_{t=(i-1)k+1}^{ik}\, \sigma_{t}\right)$ for all $i \in \{1,2 \ldots, d\}$. Note that the sequence $(\tilde{\sigma}_1, 
    \ldots,\tilde{\sigma}_d)$ gives a path down the tree $\Tcal$. The game proceeds as follows. The adversary samples a string $\sigma\in \{ \pm 1\}^T$ uniformly at random and generates a sequence of labeled instances $(x_1, S_1), \ldots (x_T, S_T)$ such that for all $i \in \{1,2, \ldots, d\}$ and all $t \in \{(i-1)k+1, \ldots, ik\} $, we have $x_t = \Tcal_{i}(\tilde{\sigma}_{< i})$ and $S_t = f_{i}((\tilde{\sigma}_{<i}, \sigma_t))$.  That is, on round $t \in \{(i-1)k+1, \ldots, ik\}$,  the adversary always reveals the instance $\Tcal_{i}(\tilde{\sigma}_{< i})$ but alternates between revealing the sets labeling the left and right outgoing edges from $\Tcal_{i}(\tilde{\sigma}_{< i})$  depending on $\sigma_t$.

    Let $\Acal$ be any randomized online learner. Then, for each block $i\in [d]$, we have
    \[\expect\left[\sum_{t=(i-1)k+1}^{ik} \indicator\left\{ \Acal(x_t) \notin S_t\right\} \right] \geq \sum_{t=(i-1)k+1}^{ik} \frac{1}{2} = \frac{k}{2}.\]
    The inequality above holds because $S_t$ is chosen uniformly at random from two disjoint sets $f_{i}((\tilde{\sigma}_{<i}, -1))$ and $f_{i}((\tilde{\sigma}_{<i}, +1))$, so the expected loss of any randomized algorithm is at least $1/2$.

    Let $h_{\tilde{\sigma}}$ be the hypothesis at the end of the path $(\tilde{\sigma}_1, \ldots, \tilde{\sigma}_d)$ in $\Tcal$. For each block $i \in [d]$, we have
    \begin{equation*}
        \begin{split}
            \expect\left[\sum_{t=(i-1)k+1}^{ik} \indicator\left\{ h_{\tilde{\sigma}}(x_t) \notin S_t\right\} \right] = \expect\left[\sum_{t=(i-1)k+1}^{ik} \indicator\{\tilde{\sigma}_{i} \neq \sigma_t\} \right] 
            &= \frac{k}{2} - \frac{1}{2} \expect\left[\sum_{t=(i-1)k+1}^{ik} \tilde{\sigma}_i \, \sigma_t \right] \\
            &= \frac{k}{2} - \frac{1}{2} \expect\left[ \,\left|\sum_{t=(i-1)k+1}^{ik}  \sigma_j \right|\, \right] \\
            &\leq \frac{k}{2} - \sqrt{\frac{k}{8}},
        \end{split}
    \end{equation*}
 where the final step follows upon using Khinchine's inequality \cite[Page 364]{cesa2006prediction}. Combining these two bounds above, we obtain
 \[\expect\left[\sum_{t=(i-1)k+1}^{ik} \indicator\left\{ \Acal(x_t) \notin S_t\right\} - \sum_{t=(i-1)k+1}^{ik}\indicator\left\{ h_{\tilde{\sigma}}(x_t) \notin S_t\right\}\right] \geq \sqrt{\frac{k}{8}}. \]
 Summing this inequality over $d$ blocks, we obtain
 \begin{equation*}
     \begin{split}
         \expect\left[\sum_{t=1}^{T} \indicator\left\{ \Acal(x_t) \notin S_t\right\} - \inf_{h \in \Hcal}\sum_{t=1}^{T}\indicator\left\{ h(x_t) \notin S_t\right\}\right] &\geq \expect\left[\sum_{t=1}^{T} \indicator\left\{ \Acal(x_t) \notin S_t\right\} - \sum_{t=1}^{T}\indicator\left\{ h_{\tilde{\sigma}}(x_t) \notin S_t\right\}\right]\\
         &\geq d\,\sqrt{\frac{k}{8}} = \sqrt{\frac{dT}{8}}.
     \end{split}
 \end{equation*}
which completes our proof.\end{proof}

\section{Applications} \label{appdx:applications}

\subsection{Online Multilabel Ranking}\label{appdx:mlr}
In this section, we prove Lemma \ref{lem:helly_rank}, establishing lower and upperbounds on Helly numbers of permutation sets. Before we prove Lemma \ref{lem:helly_rank}, we define some new notation. For any bit string $r \in \mathcal{R}$, let $P(r) := \{i: r^i = 1\}$ and let $|r| := |P(r)|$ denote the number of 1's. Given two bit strings $r_1, r_2$ where $|r_1| \geq |r_2|$, we say that $r_2 \subseteq r_1$ iff $P(r_2)\subseteq P(r_1)$. The following property will also be useful. Let $r_1, r_2 \in \mathcal{R}$ and without loss of generality suppose $|r_1| \geq |r_2|$. If $\mathcal{Y}(r_1) \cap \mathcal{Y}(r_2) \neq \emptyset$ then $r_2 \subseteq r_1$. To prove the contraposition, suppose that  $r_2 \nsubseteq r_1$. Then, there exist an index $j \in [K]$ such that $r_2^j = 1$ but $r_1^j = 0$. Thus, every permutation in $\mathcal{Y}(r_2)$ ranks label $j$ in the top $|r_2|$, but every permutation in $\mathcal{Y}(r_1)$ ranks label $j$ outside the top $|r_1|$. That is, for all $\pi_2 \in \mathcal{Y}(r_2)$ we have $\pi_2^j \leq |r_2|$ but for all $\pi_1 \in \mathcal{Y}(r_1)$, we have $\pi_1^j > |r_1|$. Since $|r_2| \leq |r_1|$, we have $\mathcal{Y}(r_1) \cap \mathcal{Y}(r_2) = \emptyset$.  We are now ready to prove the main claim. At a high-level, our proof exploits the fact that if we have a sequence of bit strings such that $r_Q \subseteq r_{Q-1} 
\subseteq  ... \subseteq r_1$, then we can iteratively construct a permutation that lies in all $\mathcal{Y}(r_i)$. 

\begin{proof}(of Lemma \ref{lem:helly_rank})
 Let $Q \geq 2$ and let $\{r_i\}_{i = 1}^Q \subseteq \mathcal{R}$ be a sequence of bit strings. It suffices to show that if for all $i, j \in [Q]$ we have $\mathcal{Y}(r_i) \cap \mathcal{Y}(r_j) \neq \emptyset$, then we have $\bigcap_{i \in [Q]} \mathcal{Y}(r_i) \neq \emptyset$. Without loss of generality, suppose $\{r_i\}_{i = 1}^Q$ is sorted in increasing order of size. That is, for all $i, j \in [Q]$ such that $i > j$, we have $|r_i| \geq |r_j|$. Then, by the property above, for all $i, j \in [Q]$ where $i > j$ we have $r_j \subseteq r_i$. We now construct a permutation $\pi: [K] \rightarrow [K]$ such that for all $i \in [Q]$, we have $\pi \in \mathcal{Y}(r_i)$.

For every $i \in \{2, ..., Q\}$, let $\phi_{i}: P(r_{i}) \setminus P(r_{i-1})\rightarrow [|r_{i}|] \setminus [|r_{i-1}|]$ denote an arbitrary bijective mapping from $P(r_{i}) \setminus P(r_{i-1})$ to $[|r_{i}|] \setminus [|r_{i-1}|]$. For $i = 1$, let $\phi_{1}: P(r_{1}) \rightarrow [|r_{1}|]$ be a bijective mapping from $P(r_{1})$ to $[|r_{1}|]$. Finally, let  $\phi_{Q+1}: [K] \setminus P(r_Q) \rightarrow [K] \setminus [|r_Q|]$ denote an arbitrary bijective mapping from $[K] \setminus P(r_Q)$ to $[K] \setminus [|r_Q|]$.  Note that by definition, for all $i, j \in \{ 1,..., Q+1\}$, the image space of $\phi_i$ and $\phi_j$ are disjoint. Moreover, the union of the image space across all bijective mappings $\phi_i$'s is $[K]$. Accordingly, we now use these bijective mappings to construct a permutation $\pi \in \mathcal{Y}$. In particular, let $\pi$ be the permutation such that for all $j \in P(r_1)$, we have $\pi^j = \phi_1(j)$,  for all $i \in \{2, ..., Q\}$ and $j \in P(r_{i}) \setminus P(r_{i-1})$, we have $\pi^j = \phi_i(j)$ , and for all $j \in [K] \setminus P(r_Q)$ we have $\pi^j = \phi_{Q+1}(j)$ . We now need to show that for all $i \in [Q]$, $\pi \in \mathcal{Y}(r_i)$.  

Fix an $i \in \mathcal{Q}$ and consider $r_i$. It suffices to show that for all $j \in P(r_i)$, we have $\pi^j \leq |r_i|$. That is, $\pi$ ranks the labels in $P(r_i)$ in the top $|r_i|$. By the subset property, we have

$$P(r_i)  = P(r_1) \cup  \bigcup_{j = 2}^{i} P(r_{j}) \setminus P(r_{j-1}).$$

Consider some $p \in P(r_i)$. Then, by the equality above, either $p \in  P(r_1)$ or $p \in \bigcup_{j = 2}^{i} P(r_{j}) \setminus P(r_{j-1})$. Suppose $p \in P(r_1)$, then by definition $\pi^p = \phi_1(p) \in [|r_1|]$ and therefore  $\pi^p \leq |r_i|$. Suppose $p \in \bigcup_{j = 2}^{i} P(r_{j}) \setminus P(r_{j-1})$. In particular, suppose $p \in P(r_{j}) \setminus P(r_{j-1}) $ for some $Q \geq j > 1$. Then by definition, $\pi^p = \phi_j(p) \in [|r_{j}|] \setminus [|r_{j-1}|]$ and therefore $\pi^p \leq |r_i|$ since $|r_j| \leq |r_i|$. This shows that for every $j \in P(r_i)$, $\pi$ ranks $j$ in the top $|r_i|$ and therefore $\ell_{\text{0-1}}(\pi, r_i) = 0$. Since $i \in [Q]$ is arbitrary,  this completes the proof as we have shown that $\bigcap_{i=1}^Q \mathcal{Y}(r_i) \neq \emptyset$. 
\end{proof}

\subsection{Ranking Littlestone dimension}
We end this section by defining an equivalent, arguably more natural, dimension that provides a tight quantitative characterization of online multilabel ranking learnability under binary relevance score feedback. The key insight is that we can actually label the edges in the $\texttt{SL}_2$ tree with bit strings instead of sets from $\mathcal{S}(\mathcal{Y})$. This intuition leads to the following dimension for online multilabel ranking.


\begin{definition}[Ranking Littlestone dimension]\label{def:rldim}
\noindent Let $\mathcal{T}$ be a complete $\mathcal{X}$-valued binary tree of depth $d$. The tree $\mathcal{T}$ is shattered by $\mathcal{H} \subseteq \Ycal^{\Xcal}$  if there exists a sequence $(f_1, ..., f_d)$ of edge-labeling functions  $f_t: \{\pm1\}^{t} \rightarrow \mathcal{R}$  such that for every path $\sigma = (\sigma_1, ..., \sigma_d) \in \{\pm1\}^d$, there exists a hypothesis $h_{\sigma} \in \mathcal{H}$ such that for all $t \in [d]$,  $\ell_{\text{0-1}}(h_{\sigma}(\mathcal{T}_t(\sigma_{<t})), f_t(\sigma_{\leq t})) = 0$, but  $f_t((\sigma_{< t}, +1)) \nsubseteq f_t((\sigma_{< t}, -1))$ and  $f_t((\sigma_{< t}, -1)) \nsubseteq f_t((\sigma_{< t}, +1))$. The Ranking Littlestone dimension of $\mathcal{H}$, denoted  $\emph{\texttt{RL}}(\mathcal{H}, \mathcal{S}(\mathcal{Y}))$, is the maximal depth of a tree $\mathcal{T}$ that is shattered by $\mathcal{H}$. If there exists shattered trees of arbitrarily large depth, we say $\emph{\texttt{RL}}(\mathcal{H}, \mathcal{S}(\mathcal{Y})) = \infty$.
\end{definition}

Since bit strings map one-to-one with sets in $\mathcal{S}(\mathcal{Y})$,  $r_1 \nsubseteq r_2, r_2 \nsubseteq r_1$ iff $ \mathcal{Y}(r_1) \cap \mathcal{Y}(r_2) = \emptyset$, and  $\ell_{\text{0-1}}(\pi, r) = 0$ iff $\pi \in \mathcal{Y}(r)$, it follows that $\texttt{SL}_2(\mathcal{H}) = \texttt{RL}(\mathcal{H})$. Corollary \ref{cor:mlrquant} immediately shows that $\texttt{RL}(\mathcal{H})$ provides a tight quantitative characterization of online multilabel ranking learnability in both the realizable and agnostic settings.

\subsection{Online Multilabel Classification}\label{appdx:mlc}

\begin{lemma}[Helly Number of Hamming Balls]\label{lem:helly_mlc}
\noindent Let $\mathcal{Y} = \{0, 1\}^K$  and $\mathcal{S}_q(\mathcal{Y}) = \{\mathcal{B}(y, q): y \in \mathcal{Y}\}$. Then, for all $ q \in [K-1]$, we have 
\[2^{q+1} \leq \emph{\texttt{H}}(\mathcal{S}_q(\mathcal{Y}) \leq \sum_{r=0}^q \binom{K}{r} + 1.\]
\end{lemma}
\begin{proof}(of Lemma \ref{lem:helly_mlc})
    Fix $q \in [K-1]$ and let $\mathcal{S}_q(\mathcal{Y}) = \{\mathcal{B}(y, q): y \in \mathcal{Y}\}$.
    To see the upperbound, observe that for any bit string $b_1 \in \{0, 1\}^K$, there are $\sum_{r=0}^q \binom{K}{r}$ sets in $\mathcal{S}_q(\mathcal{Y})$ which contain $b_1$. This follows from the fact that $b_1 \in \mathcal{B}(b_2, q)$ if and only if $b_2 \in \mathcal{B}(b_1, q)$. Therefore, $|\{A \in \mathcal{S}_q(Y): b_1 \in A\}| = |\mathcal{B}(b_1, q)| = \sum_{r=0}^q \binom{K}{r}$. The upperbound on $\texttt{H}(\mathcal{S}_q(\mathcal{Y}))$ then follows from the fact that every sequence of sets of size at least $\sum_{r=0}^q \binom{K}{r} + 1$ must have an empty intersection.
 
 To establish the lowerbound, it suffices to construct a family of $2^{q+1}$ Hamming balls that have an empty intersection, but every subfamily of size $2^{q+1}-1$ has a common element.
Let $S = \{y_1, \ldots, y_{2^{q+1}}\} \subset \{0,1\}^K$  be a family of bitstrings that embeds a hypercube of size $q+1$ and is $0$ everywhere else. That is, there exists a set of indices $I \subset [K]$ of size $|I| = q +1$ such that 
$S_{\mid I} = \{0,1\}^{q+1}$ and $S_{\mid\, [K]\backslash I} = 00\ldots 00 $
, where $S_{\mid I}$  denotes the restriction of bitstrings in $S$ to indices in $I$. We will first show that
\[\bigcap_{i=1}^{2^{q+1}} \, B(y_i, q) = \emptyset. \]
To see why this is true, pick a $y \in \{0,1\}^{K}$. Since $S$ embeds a boolean cube in $I$, there exists $i, j \in [2^{q+1}]$ such that $y_{\mid I} = y_{i\, \mid I }$ and $\neg y_{\mid I} = y_{j \, \mid I}$, where $\neg y$ is obtained by flipping every bit in $y$. Given that $|I| = q+1$, we have $\ell_{H}(y, y_j) \geq q+1$ and thus  $y \notin B(y_j, q)$. Since $y \in \{0,1\}^K$ is arbitrary,  $\bigcap_{i=1}^{2^{q+1}} \, B(y_i, q) = \emptyset$.

Next, we will show that for every $j \in [2^{q+1}]$, we have
\[\bigcap_{i \neq j} B(y_i, q) \neq \emptyset. \]
\noindent For each $y_j \in S$, define $\tilde{y}_j \in \{0,1\}^{K}$ such that $ \tilde{y}_{j \, \mid I} = \neg y_{j \, \mid I}$ and $\tilde{y}_{j \,\mid [K] \backslash I} =  00 \ldots 00 =y_{j \,\mid [K] \backslash I} $. Recall that a ball of radius $q$ centered at a vertex $v$ of a $q+1$ dimensional boolean cube contains all vertices except $\neg v$. Thus, $y_i \in B(\tilde{y}_j, q)$ for all $i \neq j$. Therefore,  $\tilde{y}_j \in \bigcap_{i \neq j} B(y_i, q)$, completing our proof. \end{proof}

As a result of Lemma \ref{lem:helly_mlc}, we do not generally have that $\texttt{SL}(\mathcal{H}) = \texttt{SL}_2(\mathcal{H})$. Accordingly, unlike multilabel ranking, the quantitative lowerbound implied by Theorem \ref{thm:agn} does not immediately follow from the structural properties in Theorem \ref{thm:relation}. Instead, Lemma \ref{lem:bsldim_lower} shows that when $K$ is sufficiently large, we are guaranteed that $\texttt{SL}_2(\mathcal{H}) > 0$ for any non-trivial hypothesis class $\mathcal{H} \subseteq \mathcal{Y}^{\mathcal{X}}$, and thus the lowerbound of Theorem \ref{thm:agn} still gives us a meaningful lowerbound scaling with $T$. 

\begin{lemma}[Lowerbound on $\texttt{SL}_2(\mathcal{H})$]\label{lem:bsldim_lower}
\noindent Fix $q \in \mathbbm{N}$ and $K \geq 2q + 1$. Let $\mathcal{Y} = \{0, 1\}^K$, $\Scal_q(\Ycal) = \{\mathcal{B}(y, q): y \in \mathcal{Y}\}$, and $\mathcal{H} \subseteq \mathcal{Y}^{\mathcal{X}}$ be a hypothesis class such that $|\mathcal{H}| \geq 2$. Then, $\emph{\texttt{SL}}_2(\mathcal{H}) \geq 1$. 
\end{lemma}{}{}

\begin{proof}(of Lemma \ref{lem:bsldim_lower})
Suppose $K \geq 2q + 1$ and $|\mathcal{H}| \geq 2$. Then, there exists a $x \in \mathcal{X}$ and a pair of hypothesis $h_1, h_2 \in \mathcal{H}$ such that $h_1(x) \neq h_2(x)$. Our goal will be to construct a shattered $\texttt{SL}_2$ tree of depth one according to Definition \ref{def:psldim} with the root node being labeled by $x$. To do so, it suffices to find two disjoint balls $S_1, S_2 \in \mathcal{S}_{q}(\mathcal{\mathcal{Y}})$ such that $h_1(x) \in S_1$ and $h_2(x) \in S_2$. We can then label the left and right outgoing edge from $x$ by $S_1$ and $S_2$ respectively.

Let $p$ denote the number of indices where $h_1(x)$ and $h_2(x)$ disagree. Note that since $h_1(x) \neq h_2(x)$, we have $p \geq 1$. Let $J \subset [K]$, $|J| = 2q + 1 - p$ denote an arbitrary subset of the indices where $h_1(x)$ and $h_2(x)$ \textit{agree}. If $2q + 1 - p$ is even,  partition $J$ into two equally sized parts $J_1$ and $J_2$. If $2q + 1 - p$ is odd, partition $J$ into $J_1$ and $J_2$ such that $|J_1| - |J_2| = 1$. For every index in $J_1$ flip the bit in the corresponding position in $h_1(x)$. Let $y_1 \in \mathcal{Y}$ be the bit string resulting from this operation. Likewise, for every index in $J_2$, flip the bit in the corresponding position in $h_2(x)$. Let $y_2 \in \mathcal{Y}$ denote the resulting bitstring. We now claim that the balls $B(y_1, q),  B(y_2, q) \in \mathcal{S}_{q}(\mathcal{Y})$ satisfy the aforementioned properties. 

First, we show that $B(y_1, q) \cap  B(y_2, q) = \emptyset$.  By construction, $y_1$ and $y_2$ differ in the locations where $h_1(x)$ and $h_2(x)$ differ plus all the indices in $J$. Thus, $\ell_H(y_1, y_2) \geq 2q + 1$. Finally, we show that $h_1(x) \in B(y_1, q)$ and $h_2(x) \in B(y_2, q)$. By construction of $y_1$ and $y_2$ and the fact that $p \geq 1$, we get that $\ell_H(h_1(x), y_1) \leq \ceil{\frac{2q + 1 - p}{2}} \leq q$ and $\ell_H(h_2(x), y_2) \leq \ceil{\frac{2q + 1 - p}{2}} \leq q$. Accordingly, we have that $h_1(x) \in B(y_1, q)$ and $h_2(x) \in B(y_2, q)$ as needed. This completes the proof as we have given two disjoint balls, $B(y_1, q)$ and $B(y_2, q)$, such that  $h_1(x) \in B(y_1, q)$ and $h_2(x) \in B(y_2, q)$. \end{proof}

Combining Lemma \ref{lem:bsldim_lower} and Theorems \ref{thm:det_real}, \ref{thm:rand_real}, and \ref{thm:agn} gives a quantitative characterization of online multilabel classification in both the realizable and agnostic settings. 

\begin{corollary}[Quantitative Online Learnability of Multilabel Classification]\label{cor:mlc}
\noindent Fix $q \in \mathbbm{N}$ and let $K \geq 2q + 1$. Let $\mathcal{Y} = \{0, 1\}^K$, $\Scal_q(\Ycal) = \{\mathcal{B}(y, q): y \in \mathcal{Y}\}$, and $\mathcal{H} \subseteq \mathcal{Y}^{\mathcal{X}}$ be a hypothesis class. Then, in the realizable setting,

$$\frac{\emph{\texttt{SL}}_2(\mathcal{H})}{2} \leq \inf_{\mathcal{A}} \, \emph{\texttt{M}}_{\mathcal{A}}(T,\mathcal{H}) \leq \emph{\texttt{SL}}(\mathcal{H}).$$

\noindent In the agnostic setting, 

$$\sqrt{\frac{\emph{\texttt{SL}}_2(\mathcal{H})\,T}{8}} \leq \inf_{\mathcal{A}}\,  \emph{\texttt{R}}_{\mathcal{A}}(T,\mathcal{H}) \leq \emph{\texttt{SL}}(\mathcal{H}) + \sqrt{2\,\emph{\texttt{SL}}(\mathcal{H})\,T \ln(T)}.$$
\end{corollary}

We leave it as an interesting future direction to get matching upper and lowerbounds for online multilabel classification.

\subsection{Online Interval Learning} 
In this section, we expand on Section \ref{sec:app} by providing one more application of set learning to a real-valued setting that we term online interval learning.  Consider an arbitrary instance space $\Xcal$, a range space $\Ycal = [-B, B]$ for some $B > 0$, and a hypothesis class $\Hcal \subseteq \Ycal^{\Xcal}$. We study an online supervised model where, in each round $t \in [T]$, the adversary reveals an example $x_t$, and the learner makes a prediction $\hat{y}_t \in [-B, B]$. The adversary then reveals an interval $[a_t, b_t]$, and the learner suffers the loss $\indicator\big\{\hat{y}_t \notin [a_t, b_t] \big\}$.  This framework models natural scenarios where the ground truth is a range of values instead of a single value. For instance, consider a model that predicts appropriate clothing size using some structural features of a customer.  Instead of one fixed size, there is usually a range of sizes that fits the customer. Since any size outside a particular range is not useful to the customer, the notion of 0-1 mistake is more natural than a regression loss.  In fact, interval-valued feedback is ubiquitous in experimental fields such as natural science and medicine because of the inherent uncertainty in measurement.

 By defining $\Scal(\Ycal) = \big\{[a, b] : -B \leq a < b \leq B\big\}$, a qualitative characterization of online interval learnability in terms of $\texttt{SL}(\Hcal)$ and $\texttt{MS}_{\gamma}(\Hcal)$ follows immediately from Theorems \ref{thm:det_real} and \ref{thm:agn}. Thus, in this section, we instead focus on establishing a quantitative characterization of online interval learnability.  As in ranking, we start by computing $\texttt{H}(\mathcal{S}(\mathcal{Y}))$. 

\begin{lemma}[Helly Number of Intervals]\label{lem:helly_interval}
\noindent Let $\Scal(\Ycal) = \big\{[a, b] \, : 
\, -B \leq a < b \leq B\big\}$. Then, $\emph{\texttt{H}}(\mathcal{S}(\mathcal{Y})) = 2$. 
\end{lemma}

Lemma \ref{lem:helly_interval} is a special case of the celebrated Helly's Theorem (see \cite{radon1921mengen, ECKHOFF1993389}). Since $\texttt{H}(\mathcal{S}(\mathcal{Y})) = 2$, by Theorem \ref{thm:relation}, we know that for all $\gamma \in [0, \frac{1}{2}]$, $\texttt{MS}_{\gamma}(\mathcal{H}) = \texttt{SL}(\mathcal{H}) = \texttt{SL}_2(\mathcal{H})$.  Therefore the $\texttt{SL}_2(\mathcal{H})$ characterizes both deterministic and randomized online interval learnability in the realizable setting. Moreover, we can use Theorems \ref{thm:det_real}, \ref{thm:rand_real}, and \ref{thm:agn} to give Corollary \ref{cor:intquant}, a sharp quantitative characterization of online interval learning in both the realizable and agnostic settings.

\begin{corollary}[Online Interval Learnability]\label{cor:intquant}
\noindent Let $\mathcal{Y} = [-B,B]$, $\Scal(\Ycal) = \Big\{[a, b] \, :\,  -B \leq a < b \leq B\Big\}$, and $\mathcal{H} \subseteq \mathcal{Y}^{\mathcal{X}}$ be a scalar-valued hypothesis class. Then, in the realizable setting,

$$\frac{\emph{\texttt{SL}}_2(\mathcal{H})}{2} \leq \inf_{\mathcal{A}} \, \emph{\texttt{M}}_{\mathcal{A}}(T,\mathcal{H}) \leq \emph{\texttt{SL}}(\mathcal{H}).$$

\noindent In the agnostic setting, 

$$\sqrt{\frac{\emph{\texttt{SL}}_2(\mathcal{H})\,T}{8}} \leq \inf_{\mathcal{A}}\,  \emph{\texttt{R}}_{\mathcal{A}}(T,\mathcal{H}) \leq \emph{\texttt{SL}}(\mathcal{H}) + \sqrt{2\,\emph{\texttt{SL}}(\mathcal{H})\,T \ln(T)}.$$

\end{corollary}

\end{document}